\documentclass[final]{cvpr}

\usepackage{times}
\usepackage{epsfig}
\usepackage{graphicx}
\usepackage{amsmath}
\usepackage{amssymb}
\usepackage{amsthm}
\usepackage{multicol}
\usepackage{multirow}
\usepackage{subcaption}
\usepackage{xfrac}
\usepackage{rotating}
\usepackage{array}

\usepackage[pagebackref=true,breaklinks=true,colorlinks,bookmarks=false]{hyperref}

\DeclareMathOperator{\sgn}{sgn}
\newtheorem{claim}{Claim}

\newcolumntype{C}[1]{>{\centering\arraybackslash}m{#1}}

\def\cvprPaperID{10123} 
\def\confYear{CVPR 2021}

\begin{document}

\title{An Alternative Probabilistic Interpretation of the Huber Loss}

\author{Gregory P. Meyer\vspace{0.25em}\\
Uber Advanced Technologies Group\\
{\tt\small gmeyer@uber.com}
}

\maketitle

\begin{abstract}
The Huber loss is a robust loss function used for a wide range of regression tasks.
To utilize the Huber loss, a parameter that controls the transitions from a quadratic function to an absolute value function needs to be selected.
We believe the standard probabilistic interpretation that relates the Huber loss to the Huber density fails to provide adequate intuition for identifying the transition point.
As a result, a hyper-parameter search is often necessary to determine an appropriate value.
In this work, we propose an alternative probabilistic interpretation of the Huber loss, which relates minimizing the loss to minimizing an upper-bound on the Kullback-Leibler divergence between Laplace distributions, where one distribution represents the noise in the ground-truth and the other represents the noise in the prediction.
In addition, we show that the parameters of the Laplace distributions are directly related to the transition point of the Huber loss.
We demonstrate, through a toy problem, that the optimal transition point of the Huber loss is closely related to the distribution of the noise in the ground-truth data.
As a result, our interpretation provides an intuitive way to identify well-suited hyper-parameters by approximating the amount of noise in the data, which we demonstrate through a case study and experimentation on the Faster R-CNN and RetinaNet object detectors.
\end{abstract}

\section{Introduction}

A typical problem in machine learning is estimating a function $F_\theta$ that maps from $x \in \mathbb{R}^n$ to $y \in \mathbb{R}$ given a set of training examples $\mathcal{D} = \{x_i, y_i\}_{i=0}^N$.
The parameters of the function $\theta$ are often determined by minimizing a loss function $\mathcal{L}$,
\begin{equation}
    \hat{\theta} = \arg\min_\theta \sum_{i=0}^N \mathcal{L}(y_i - F_\theta(x_i))
\end{equation}
and the choice of loss function can be crucial to the performance of the model.
The Huber loss is a robust loss function that behaves quadratically for small residuals and linearly for large residuals \cite{huber}.
The loss function was proposed over a half-century ago, and it is still widely used today for a variety of regression tasks, including 2D object detection \cite{fast-rcnn,retinanet,ssd,faster-rcnn}, 3D object detection \cite{chen2016,mv3d,avod,pixor}, shape and pose estimation \cite{alp2017,ku2019,wang2018}, and stereo estimation \cite{chang2018}.

A challenge with utilizing the Huber loss in practice is selecting an appropriate value to transition from a quadratic error to a linear error.
Under certain assumptions, minimizing a loss function can be interpreted as maximizing the likelihood of $y_i$ given $x_i$,
\begin{equation}
    \hat{\theta} = \arg\max_\theta \prod_{i=0}^N p(y_i|x_i, \theta)
\end{equation}
when $p(y_i|x_i, \theta) \propto \exp\left[-\mathcal{L}(y_i - F_\theta(x_i))\right]$.
Therefore, the estimate $\hat{\theta}$ that minimizes the Huber loss can be interpreted as the maximum likelihood estimate of $\theta$ when $p(y_i|x_i, \theta)$ is the Huber density~\cite{huber}.
The Huber density can be difficult to interpret; as a result, hyper-parameter search is often employed to identify a satisfactory transition point for the Huber loss.

In this work, we propose an alternative probabilistic interpretation of the Huber loss.
Our interpretation assumes $y_i$ is a noisy estimate of the true value $y_i^*$, and we show that minimizing the Huber loss is equivalent to minimizing an upper-bound on the Kullback-Leibler (KL) divergence,
\begin{equation}
    \sum_{i=0}^N D\left(p(y^*_i|y_i)\|q(y^*_i|x_i, \theta)\right)
\end{equation}
when $p(y^*_i|y_i)$ and $q(y^*_i|x_i, \theta)$ are Laplace distributions and the scale of the distributions are directly related to the transition point of the Huber loss.
For real-world problems, the value of $y_i$ corresponding to $x_i$ is often provided by a human annotator; therefore, it is likely to contain some amount of noise. 
We believe that approximating the amount of noise in the ground-truth is a more intuitive way to determine the transition point for the Huber loss than reasoning about the Huber density.

In the following sections, we survey the related work (Section~\ref{sec:related-work}), review the Huber loss and maximum likelihood estimation in detail (Section~\ref{sec:background}), propose our alternative probabilistic interpretation of the Huber loss (Section~\ref{sec:method}), utilize a toy problem to illustrate the relationship between the optimal transition point of the Huber loss and the noise distribution of the ground-truth (Section~\ref{sec:toy-problem}), leverage our interpretation to analyze the loss functions utilized by modern object detectors (Section~\ref{sec:case-study}), and show that our proposed interpretation can lead to better hyper-parameters (Section~\ref{sec:experiments}).

\section{Related Work}
\label{sec:related-work}

Noy and Crammer \cite{noy2014}, remarked on the similarity between the Huber loss and the KL divergence of Laplace distributions, which motivates their use of a Laplace-like family of distributions in the PAC-Bayes framework.
However, they did not explore the relationship beyond this observation.
In this work, we further pursue the connection between the Huber loss and the KL divergence of Laplace distributions, and we identify the links between the parameters of the Huber loss and the parameters of the Laplace distributions.

Lange \cite{lange1990}, proposed a set of potential functions for image reconstruction that behave like the Huber loss, but unlike the Huber loss, these functions are more than once differentiable.
In this work, we propose a loss function which is similar to a potential function in \cite{lange1990}.
However, our proposed loss is derived directly from the KL divergence of Laplace distributions; whereas, the potential functions in \cite{lange1990} are derived through double integration of symmetric and positive functions.

\section{Background}
\label{sec:background}

\subsection{Huber Loss}

Loss functions commonly used for regression are $L_1(x) = |x|$ and $L_2(x) = \frac{1}{2} x^2$.
Both of these functions have advantages and disadvantages; $L_1$ is less sensitive to outliers in the data, but it is not differentiable at zero.
Whereas, the $L_2$ is differentiable everywhere, but it is highly sensitive to outliers.
Huber proposed the following loss as a compromise between the $L_1$ and $L_2$ losses \cite{huber}:
\begin{equation}
    H_\alpha(x) =
    \begin{cases} 
        \frac{1}{2} x^2, & |x| \leq \alpha \\
        \alpha \left( |x| - \frac{1}{2} \alpha \right), & |x| > \alpha
    \end{cases}
    \label{eqn:huber_loss}
\end{equation}
where $\alpha \in \mathbb{R}^+$ is a positive real number that controls the transition from $L_1$ to $L_2$.
The Huber loss is both differentiable everywhere and robust to outliers.

A disadvantage of the Huber loss is that the parameter $\alpha$ needs to be selected.
In this work, we propose an intuitive and probabilistic interpretation of the Huber loss and its parameter $\alpha$, which we believe can ease the process of hyper-parameter selection.
Next, we review how minimizing the loss functions are related to maximum likelihood estimation.

\subsection{Maximum Likelihood Estimation}

Assume we have some data $\mathcal{D} = \{x_i, y_i\}_{i=0}^N$ independently drawn from some unknown distribution.
Let us model the relationship between $x_i$ and $y_i$ as
\begin{equation}
    y_i = F_\theta(x_i) + \epsilon
\end{equation}
where $F_\theta$ is a deterministic function parameterized by $\theta$, and $\epsilon$ is random noise drawn from some known distribution.
The goal of maximum likelihood estimation is to identify the parameter $\hat{\theta}$ that maximizes the likelihood of $y_i$ given $x_i$ across the dataset $\mathcal{D}$.
Note that maximizing the likelihood of $y_i$ given $x_i$ is equivalent to minimizing the negative log likelihood,
\begin{align}
  \begin{split}
    \hat{\theta} &= \arg\max_\theta \prod_{i=0}^N p(y_i|x_i, \theta) \\
    &= \arg\min_\theta -\sum_{i=0}^N \log p(y_i|x_i, \theta)\text{.}
  \end{split}
\end{align}
Consider the case when the noise $\epsilon$ is drawn independently from a zero-mean Gaussian distribution.
The probability density for $y_i$ given $x_i$ becomes
\begin{equation}
    p(y_i|x_i, \theta) = \frac{1}{\sqrt{2 \pi \sigma^2}} \exp\left(-\frac{(y_i - F_\theta(x_i))^2}{2\sigma^2}\right)
\end{equation}
where $\sigma \in \mathbb{R^+}$ is the standard deviation of the noise, and the negative log likelihood becomes
\begin{equation}
    -\log p(y_i|x_i, \theta) = \log \sqrt{2 \pi \sigma^2} + \frac{(y_i - F_\theta(x_i))^2}{2\sigma^2}\text{.}
\end{equation}
Notice that
\begin{align}
  \begin{split}
    \hat{\theta} &= \arg\min_\theta -\sum_{i=0}^N \log p(y_i|x_i, \theta) \\
    &= \arg\min_\theta \sum_{i=0}^N (y_i - F_\theta(x_i))^2
  \end{split}
\end{align}
by assuming a constant $\sigma$ and dropping the constant term.
Therefore, identifying $\hat{\theta}$ that minimizes the $L_2$ loss over the dataset is equivalent to the maximum likelihood estimate of $\theta$ when $p(y_i|x_i, \theta)$ follows a Gaussian distribution.
In addition, minimizing the $L_1$ loss can be shown to be the same as the maximum likelihood estimation when the noise is drawn from a Laplace distribution.
In \cite{huber}, it is demonstrated that minimizing the Huber loss provides the maximum likelihood estimate when the probability density takes the form $p(y_i|x_i, \theta) \propto \exp\left[-H_\alpha(y_i - F_\theta(x_i))\right]$, which is sometimes referred to as the Huber density. 
The Huber loss is a combination of the $L_1$ and $L_2$ losses; therefore, the Huber density is a hybrid of the Gaussian and Laplace distributions.

The Huber density is more complicated than either the Gaussian or Laplace distribution individually, and we believe this complexity makes it challenging to use this interpretation of the Huber loss for selecting the parameter $\alpha$.
For this reason, we propose an alternative probabilistic interpretation.

\section{Proposed Method}
\label{sec:method}

Like above, assume we have a dataset $\mathcal{D} = \{x_i, y_i\}_{i=0}^N$, but let us consider the following relationships:
\begin{align}
    y^*_i &= y_i + \epsilon_1 \\
    y^*_i &= F_\theta(x_i) + \epsilon_2
\end{align}
where $y^*_i$ is an unknown value we would like to estimate with $F_\theta(x_i)$, $y_i$ is a known estimate of $y^*_i$, and $\epsilon_1$ and $\epsilon_2$ are random noise variables drawn independently from separate but known distributions.
Since $y^*_i$ is hidden, we are unable to estimate $\hat{\theta}$ by directly maximizing the likelihood of $y^*_i$ given $x_i$.
Alternatively, we can estimate $\hat{\theta}$ by minimizing the Kullback-Leibler (KL) divergence between the distributions $p(y^*_i|y_i)$ and $q(y^*_i|x_i, \theta)$.
Intuitively, $p(y^*_i|y_i)$ represents our uncertainty in the label $y_i$, and $q(y^*_i|x_i, \theta)$ represents our uncertainty in the model's prediction $F_\theta(x_i)$.
Also, note that minimizing the KL divergence is equivalent to minimizing the cross entropy, 
\begin{align}
  \begin{split}
    \hat{\theta} &= \arg\min_\theta \sum_{i=0}^N D\left(p(y^*_i|y_i)\|q(y^*_i|x_i, \theta)\right) \\
    &= \arg\min_\theta -\sum_{i=0}^N \left(\int_{-\infty}^\infty p(y^*_i|y_i)\log q(y^*_i|x_i, \theta) dy^*_i \right)
  \end{split}
\end{align}
since the entropy of $p(y^*_i|y_i)$ is constant.
If $p(y^*_i|y_i)$ is a Dirac delta function centered on $y_i$, i.e. the label contains zero noise, minimizing the cross entropy is equivalent to minimizing the negative log likelihood of $q(y_i|x_i, \theta)$.
Therefore, finding $\hat{\theta}$ by minimizing the KL divergence is exactly the maximum likelihood estimate of $\theta$ when $y^*_i = y_i$.

Let us assume both the labels and the predictions are contaminated with outliers, i.e. both $\epsilon_1$ and $\epsilon_2$ are drawn from Laplace distributions.
The corresponding probability densities are
\begin{equation}
    p(y^*_i|y_i) = \frac{1}{2b_1} \exp\left(-\frac{|y^*_i - y_i|}{b_1}\right)
\end{equation}
and
\begin{equation}
    q(y^*_i|x_i, \theta) = \frac{1}{2b_2} \exp\left(-\frac{|y^*_i - F_\theta(x_i)|}{b_2}\right)
\end{equation}
where $b_1 \in \mathbb{R}^+$ and $b_2 \in \mathbb{R}^+$ define the scale of the label uncertainty and prediction uncertainty, respectively.
Furthermore, the KL divergence becomes
\begin{align}
  \begin{split}
    &D\left(p(y^*_i|y_i)\|q(y^*_i|x_i, \theta)\right) \\
    &\quad = \frac{b_1 \exp\left(-\frac{\left|y_i - F_\theta(x_i)\right|}{b_1} \right) + \left|y_i - F_\theta(x_i)\right|}{b_2} + \log\frac{b_2}{b_1} - 1
  \end{split}
  \label{eqn:kld}
\end{align}
by integrating over all values of $y^*_i$.
For a derivation, please refer to Appendix~\ref{sec:derivation}.
In the following sections, we propose a loss derived from the KL divergence of Laplace distributions, show that it is related to the Huber loss, and use the relationship to gain further insight into the Huber loss.

\subsection{Proposed Loss Function}
\label{sec:proposed-loss}

We propose the following loss function:
\begin{equation}
    D_{\alpha,\beta}(x) = \frac{\alpha\exp\left(-\frac{|x|}{\alpha}\right) + |x| - \alpha}{\beta}
    \label{eqn:kl_loss}
\end{equation}
which is derived from the KL divergence of Laplace distributions (Equation \eqref{eqn:kld}) by removing the existing constant terms and by adding a new constant term to ensure the minimum value is always zero.
The variable $x$ is equal to the difference in the means of the Laplace distributions.
The parameter $\alpha \in \mathbb{R}^+$ directly corresponds to the scale of the noise in the label ($b_1$), and $\beta \in \mathbb{R}^+$ corresponds to the scale of the noise in the prediction ($b_2$).
As a result, the parameters, $\alpha$ and $\beta$, have an intuitive and probabilistic interpretation related to the variance of the Laplace distributions.
Since our modification to Equation \eqref{eqn:kld} simply removes the constant penalty for the mismatch in the standard deviation of the distributions, our loss function is identical to the KL divergence when $b_1 = b_2 = \alpha = \beta$.

\subsection{Relationship to the Huber Loss}
\label{sec:relationship}

To demonstrate the relationship between our proposed loss and the Huber loss, let us start by considering the behavior of our loss function when $|x|$ is small with respect to $\alpha$.
The second-order approximation of Equation \eqref{eqn:kl_loss} about zero is
\begin{equation}
    D_{\alpha,\beta}(x) \approx D_{\alpha,\beta}(0) + D'_{\alpha,\beta}(0) x + \frac{D''_{\alpha,\beta}(0)}{2} x^2 = \frac{1}{2 \alpha \beta} x^2\text{.}
\end{equation}
Refer to Appendix~\ref{sec:derivatives}, for a derivation of the derivatives and a proof of their existence at $x = 0$.
Furthermore, when $|x|$ is large with respect to $\alpha$ the exponential term in Equation \eqref{eqn:kl_loss} goes to zero,
\begin{equation}
    D_{\alpha,\beta}(x) \approx \frac{|x| - \alpha}{\beta}\text{.}
\end{equation}
As a result, Equation \eqref{eqn:kl_loss} can be approximated using the following piecewise function:
\begin{equation}
    D_{\alpha,\beta}(x) \approx 
    \begin{cases}
        \frac{1}{2 \alpha \beta} x^2, & |x| \leq \alpha \\
        \frac{|x| - \alpha}{\beta}, & |x| > \alpha\text{.}
    \end{cases}
\end{equation}
Like the Huber loss, our proposed loss behaves quadratically when the residual is small and linearly when the residual is large.
In addition, the following configurations tightly bound the Huber loss:
\begin{equation}
    D_{\alpha,\sfrac{1}{\alpha}}(x) \leq H_\alpha(x) \leq D_{\sfrac{\alpha}{2},\sfrac{1}{\alpha}}(x)\text{.}
\end{equation}
The relationship between the loss functions is illustrated in Figure~\ref{fig:losses}, and a formal proof of the bounds is provided in Appendix~\ref{sec:proofs}.

Minimizing the Huber loss with parameter $\alpha$ is equivalent to minimizing an upper-bound on the KL divergence of two Laplace distributions when the scale of the label distribution $b_1 = \alpha$, and the scale of the prediction distribution $b_2 = \sfrac{1}{\alpha}$.
Conversely, minimizing the KL divergence of two Laplace distributions with $b_1 = \sfrac{\alpha}{2}$ and $b_2 =\sfrac{1}{\alpha}$ is equivalent to minimizing an upper-bound on the Huber loss with parameter $\alpha$.
We believe this alternative probabilistic interpretation of the Huber loss provides significant insight into the parameter $\alpha$, which we demonstrate in the remaining sections.

\begin{figure*}
\begin{tabular}{C{3em}C{21.75em}C{21.75em}}
  \begin{tabular}{l}$\alpha = 0.1$\end{tabular} & \includegraphics[width=\linewidth]{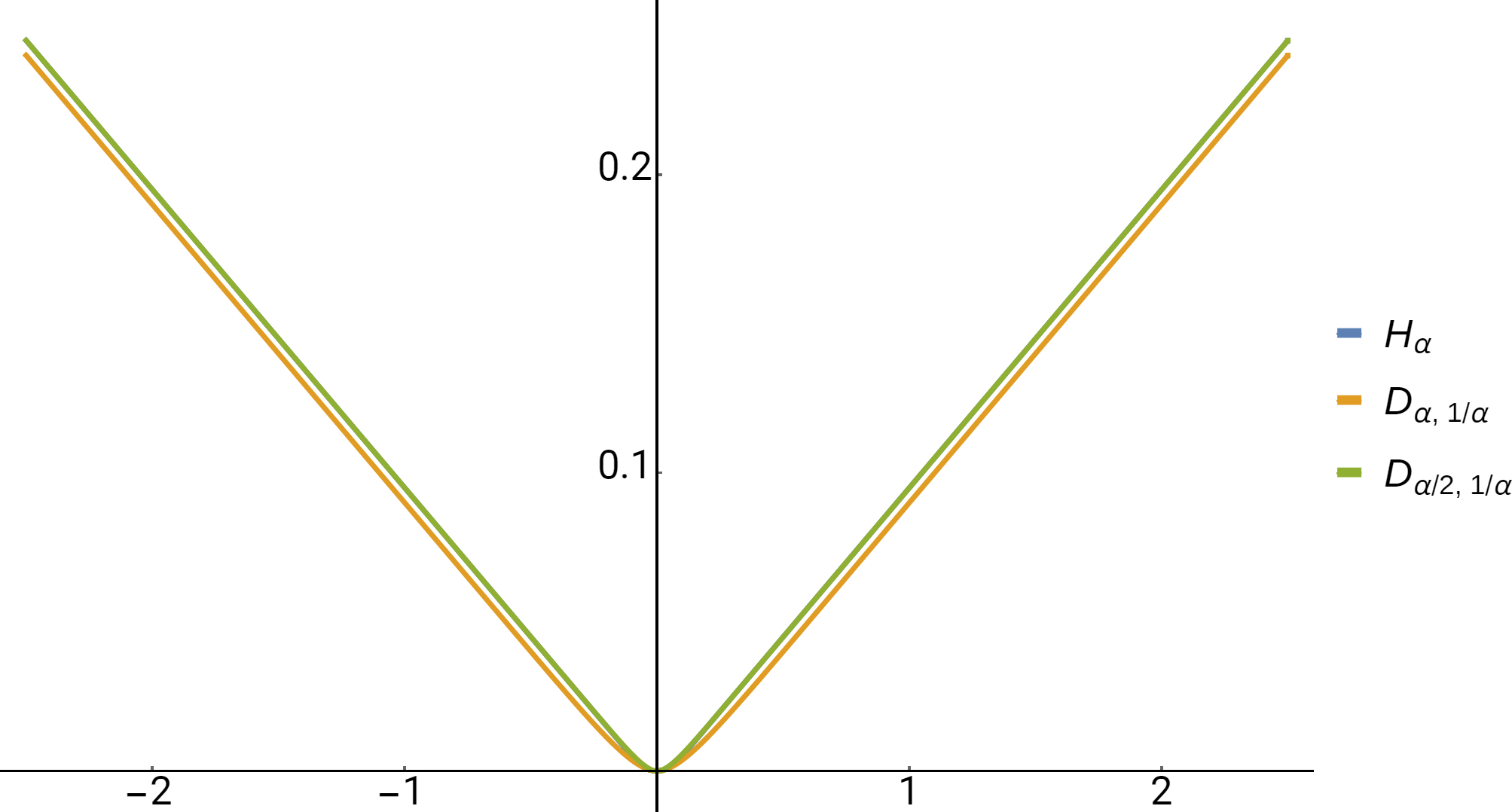} & \includegraphics[width=\linewidth]{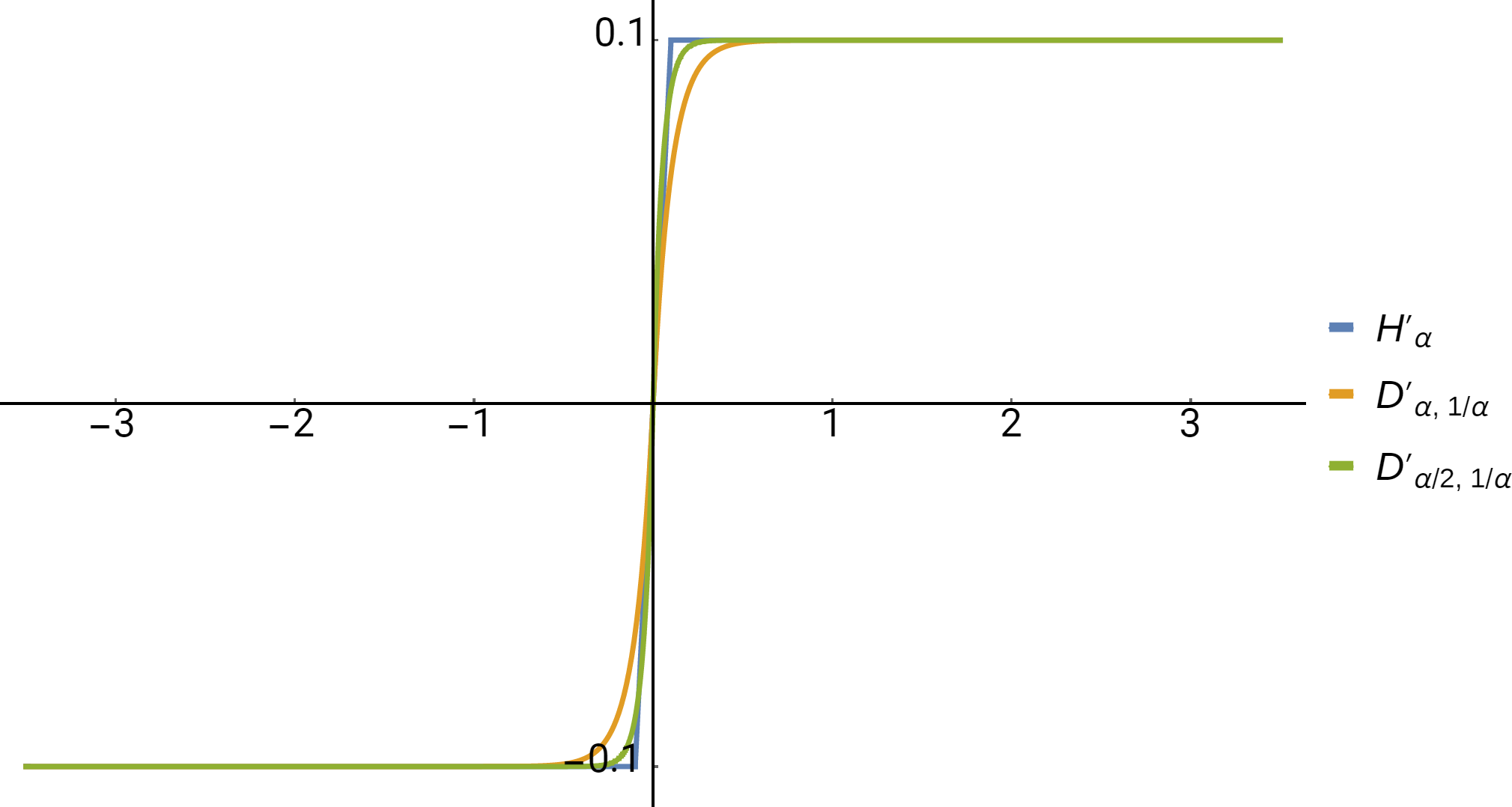} \\
  \begin{tabular}{l}$\alpha = 1$\end{tabular} & \includegraphics[width=\linewidth]{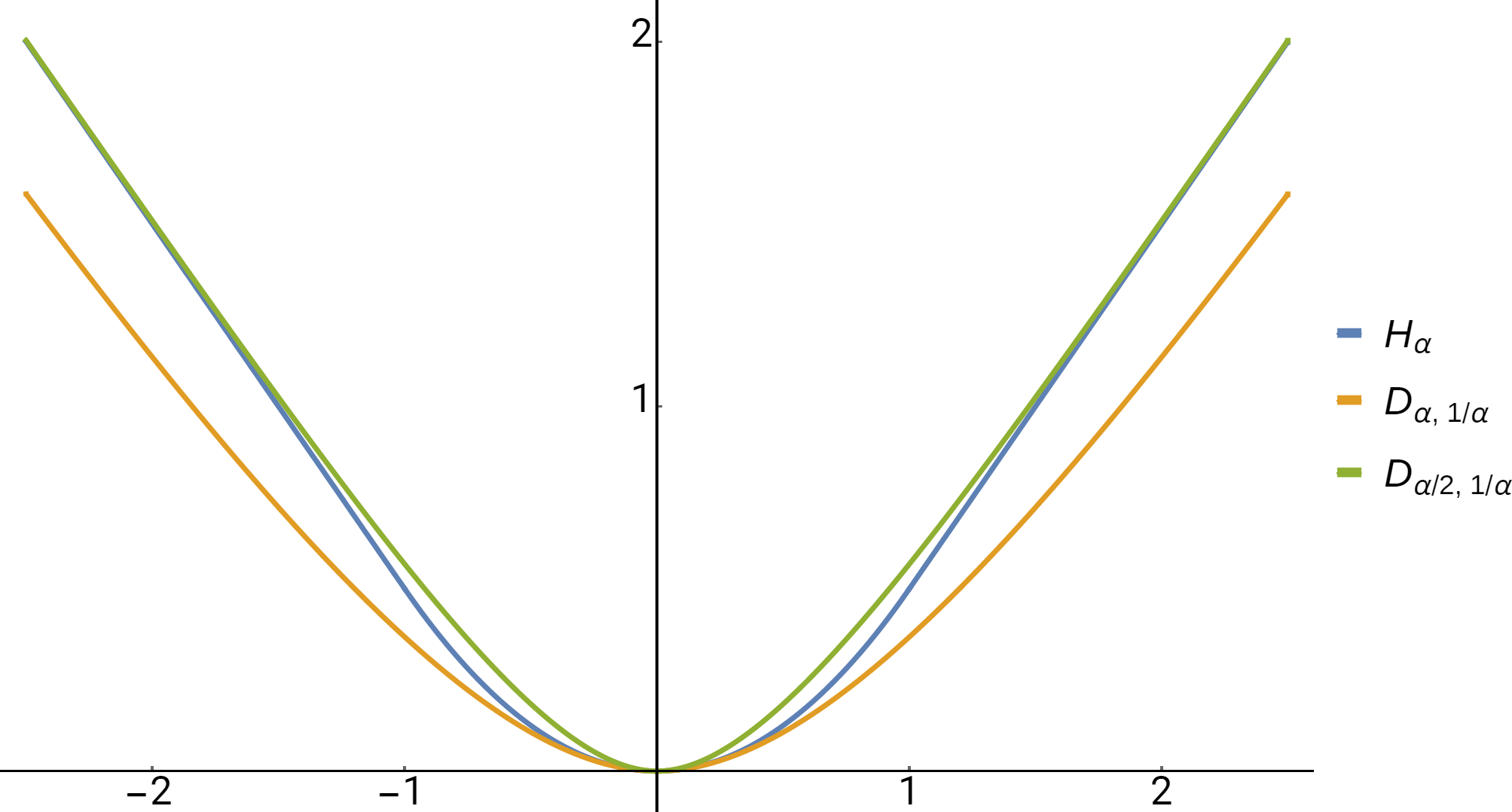} & \includegraphics[width=\linewidth]{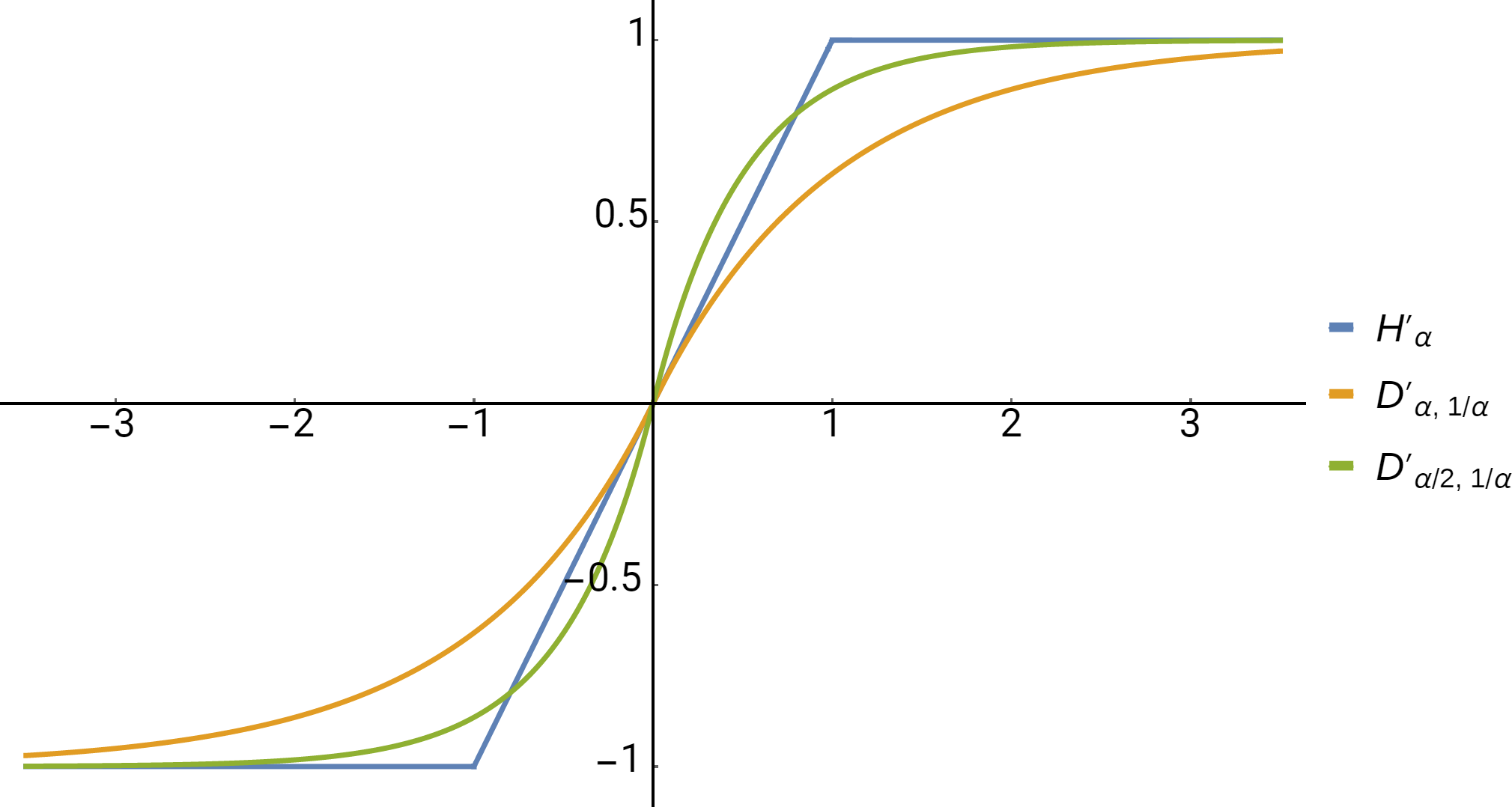} \\
  \begin{tabular}{l}$\alpha = 10$\end{tabular} & \includegraphics[width=\linewidth]{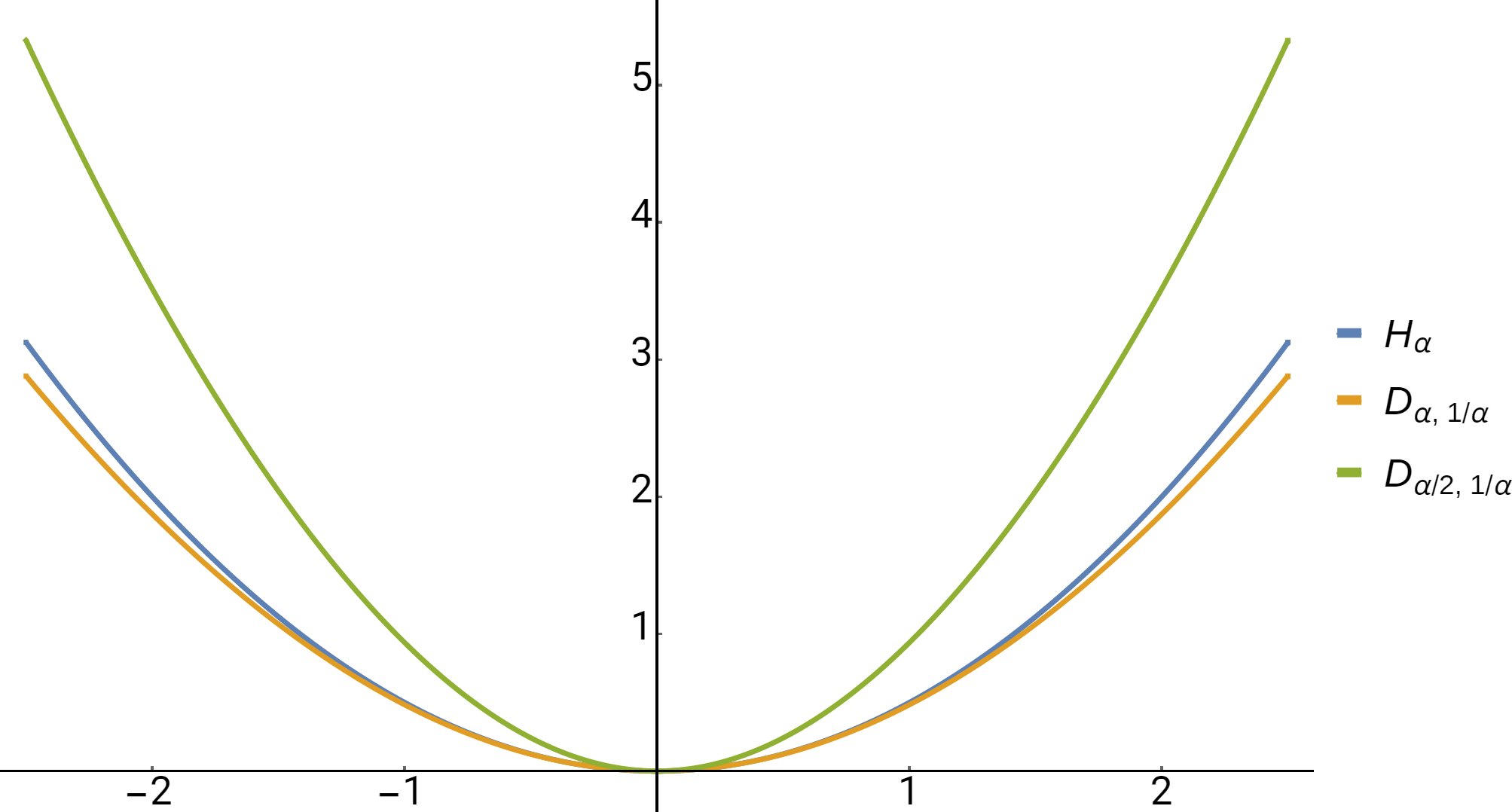} & \includegraphics[width=\linewidth]{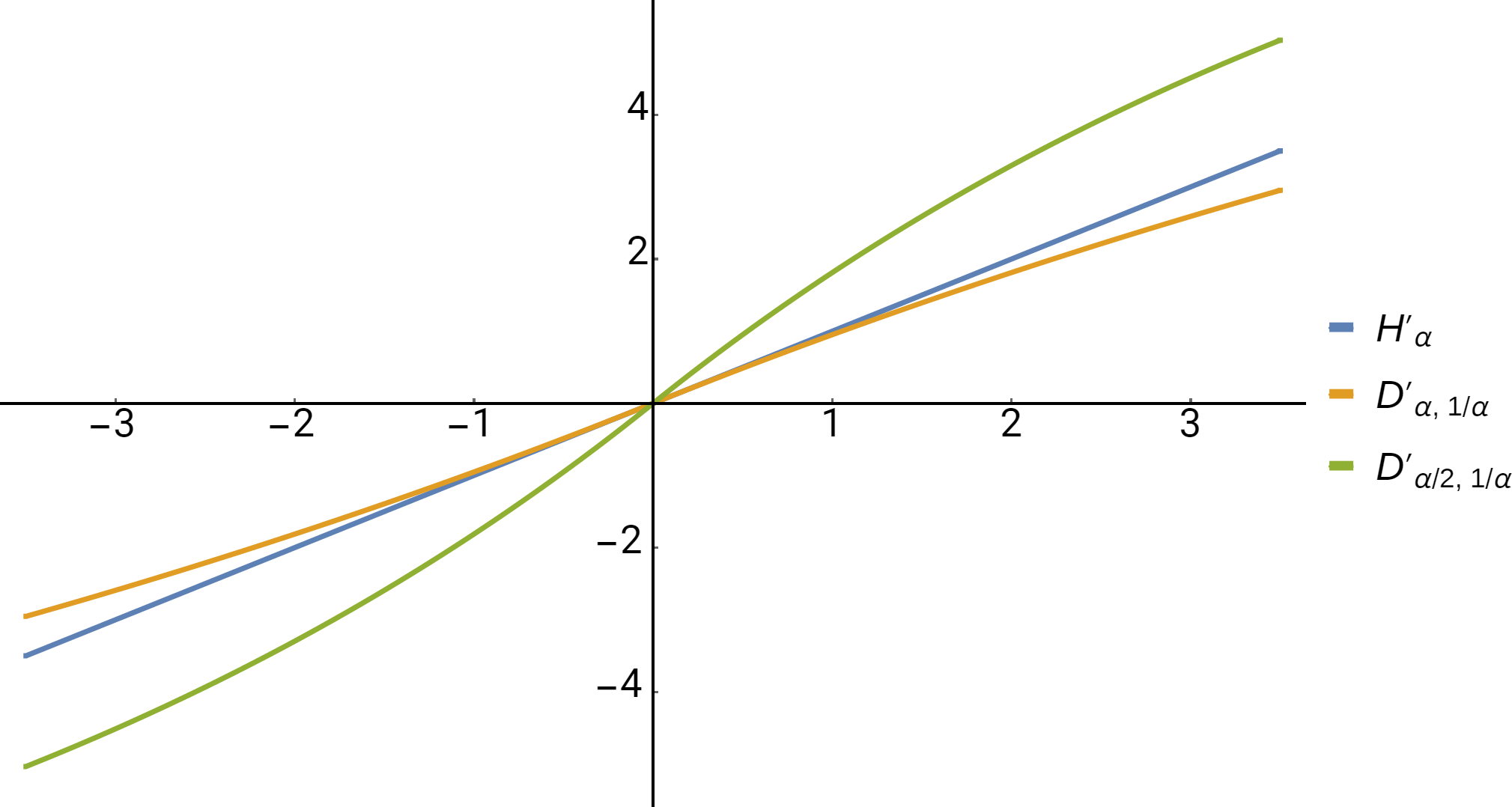} \\
  & Loss Functions & Derivatives
\end{tabular}
\caption{A comparison between the Huber loss ($H_\alpha$) and our proposed loss ($D_{\alpha,\beta}$) derived from the KL divergence of Laplace distributions. The loss function $H_\alpha$ is lower-bounded by $D_{\alpha, \sfrac{1}{\alpha}}$ and upper-bounded by $D_{\sfrac{\alpha}{2}, \sfrac{1}{\alpha}}$. The left column depicts the loss functions and the right column visualizes their derivatives. In addition, each row of the figure depicts a different set of hyper-parameters.}
\label{fig:losses}
\vspace{1em}
\end{figure*}

\subsection{Properties of the Loss Function}

Notice that scaling $x$ by a positive real number, $\gamma \in \mathbb{R}^+$, is equivalent to $D_{\alpha,\beta}(\gamma x) = D_{\sfrac{\alpha}{\gamma},\sfrac{\beta}{\gamma}}(x)$ whereas scaling the loss by $\lambda \in  \mathbb{R}^+$ is equivalent to $\lambda D_{\alpha,\beta}(x) = D_{\alpha,\sfrac{\beta}{\lambda}}(x)$.
Both of these properties are trivial to show through algebraic manipulation. 

In the following sections, we will analyze the Huber loss with the approximation $H_\alpha(x) \approx D_{\alpha,\sfrac{1}{\alpha}}(x)$. 
Combining the above properties with the approximation, we observe that
\begin{equation}
    \lambda H_\alpha(\gamma x) \approx D_{\sfrac{\alpha}{\gamma},\sfrac{1}{\alpha\gamma\lambda}}(x)\text{.}
    \label{eqn:property}
\end{equation}
Therefore, scaling the input to the Huber loss is equivalent to inversely scaling the label and prediction distributions, and scaling the output is equivalent to inversely scaling the prediction distribution.

\section{Toy Problem: Polynomial Fitting}
\label{sec:toy-problem}

\begin{figure*}
\centering
\begin{subfigure}{0.32\textwidth}
  \centering
  \includegraphics[width=\linewidth]{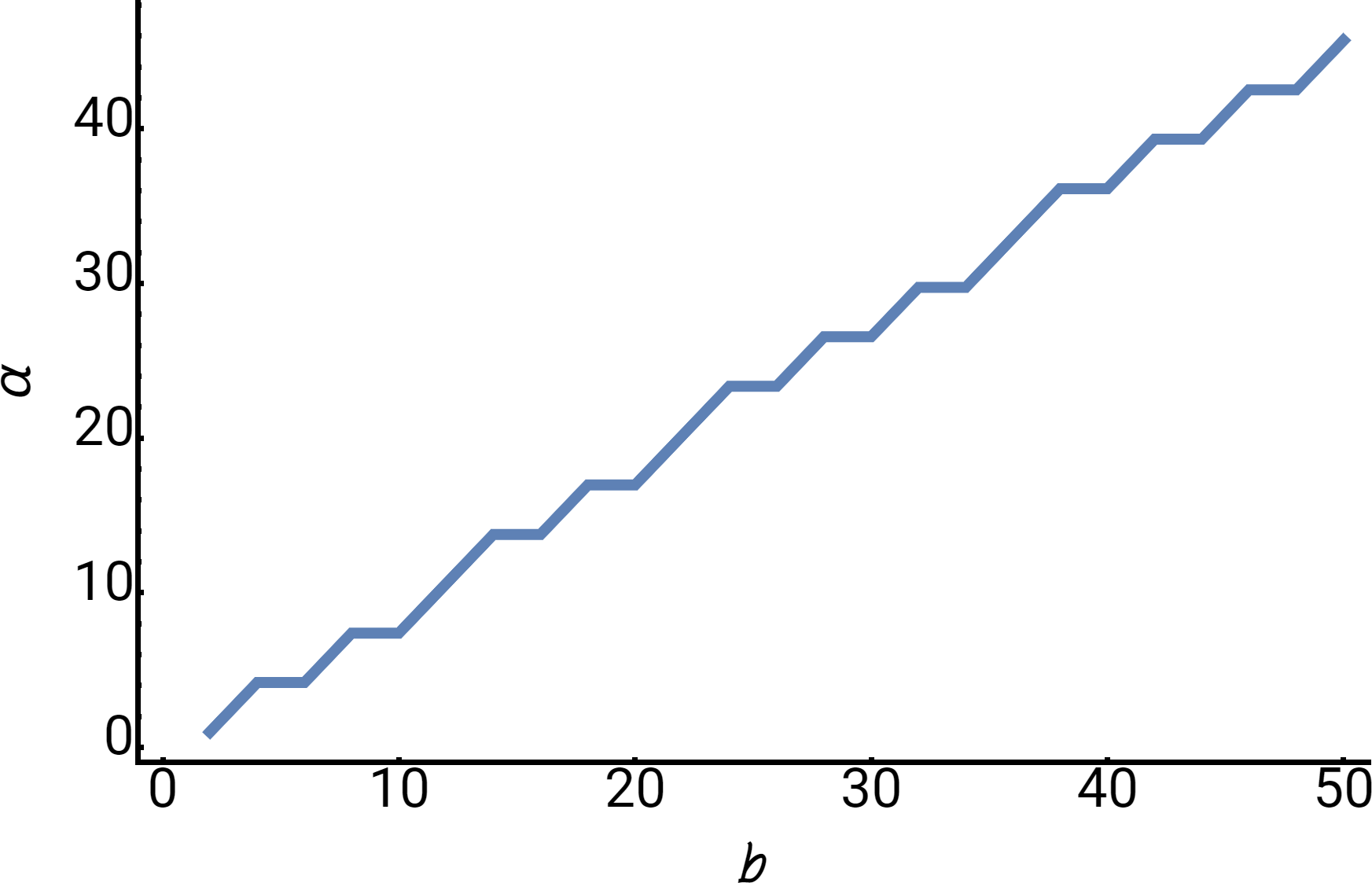}
  \caption{Laplace Distribution}
\end{subfigure}
\begin{subfigure}{0.32\textwidth}
  \centering
  \includegraphics[width=\linewidth]{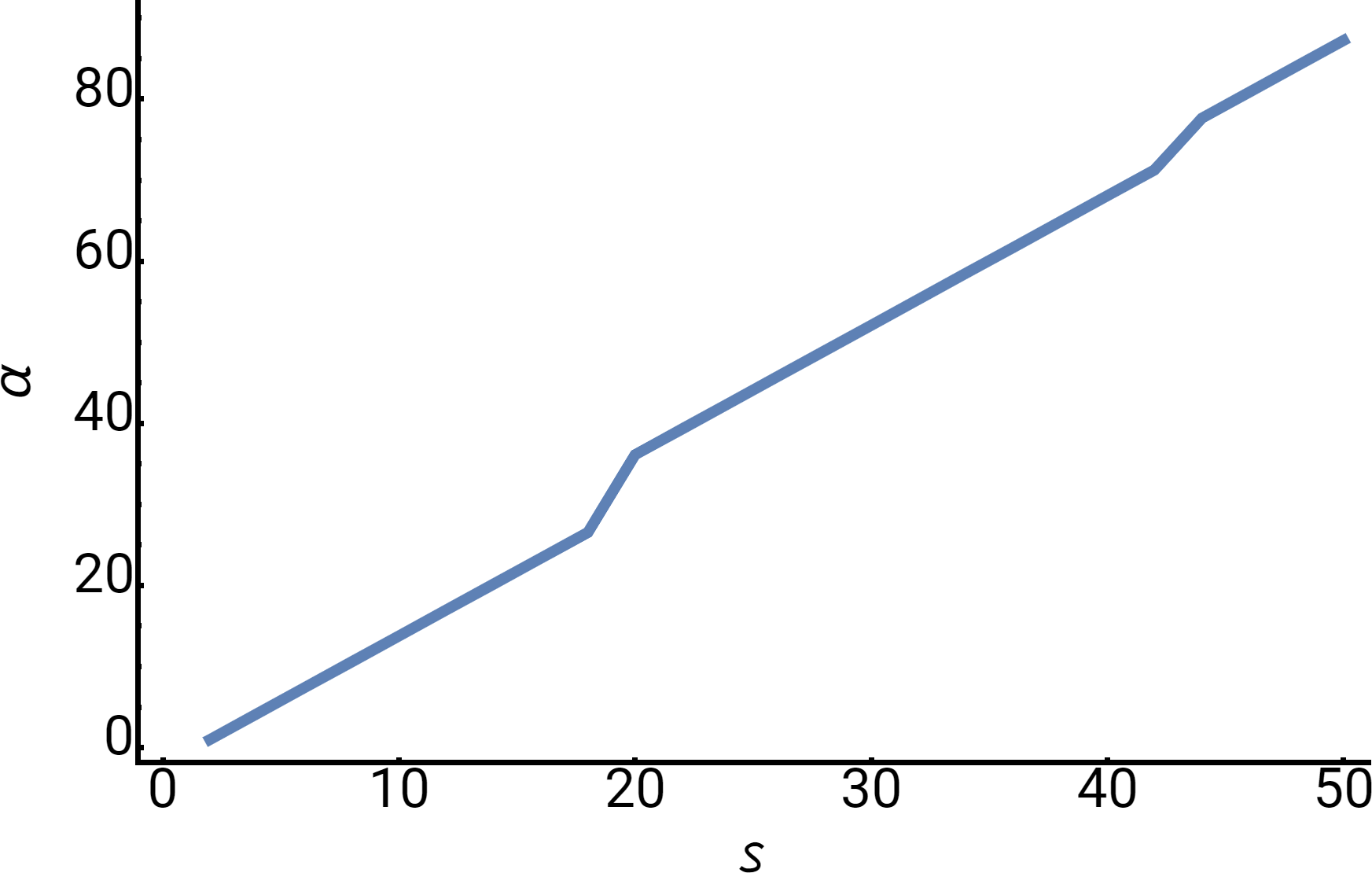}
  \caption{Logistic Distribution}
\end{subfigure}
\begin{subfigure}{0.32\textwidth}
  \centering
  \includegraphics[width=\linewidth]{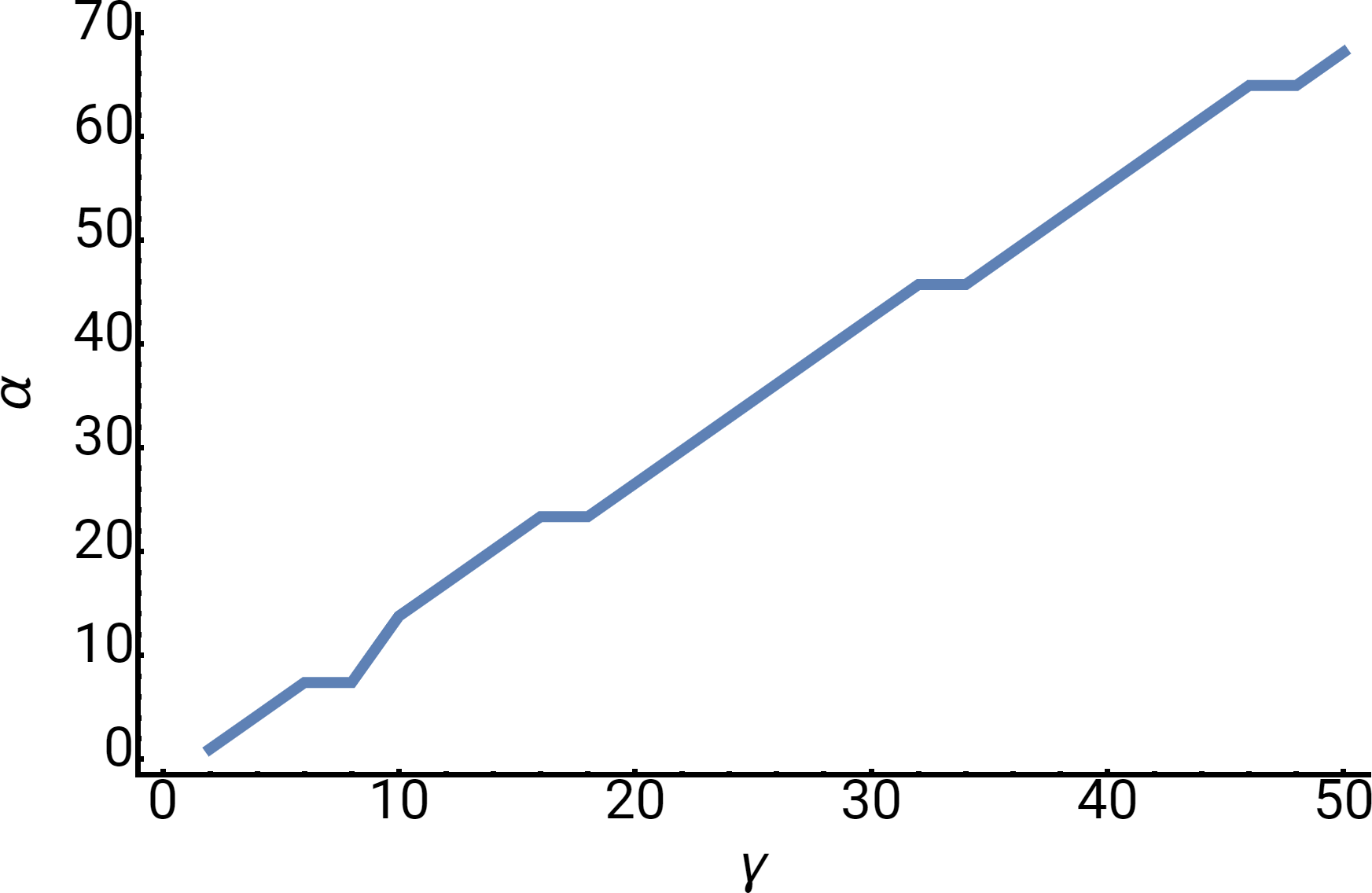}
  \caption{Cauchy Distribution}
\end{subfigure}
\caption{The results of our toy problem. The $x$-axis is the parameter that controls the scale of the respective noise distributions and the y-axis is the optimal $\alpha$ parameter for the Huber loss. For each distribution, there is an approximate linear relationship between the noise and the optimal parameter.}
\label{fig:toy-problem}
\end{figure*}

Our proposed interpretation of the Huber loss suggests there is a relationship between the parameter of the loss and the scale of the label noise.
We would like to determine whether or not knowing this relationship and having a good estimate for the noise can help us select a good parameter for the loss.
We employ a toy problem, where we control the amount of label noise, to show that the optimal $\alpha$ parameter is closely related to the scale of the label noise.

For our toy problem, we fit a one-dimensional polynomial to a set of sample points.
To create our training examples, $\mathcal{D} = \{x_i, y_i\}_{i=0}^N$, we randomly sample $x_i \in [-\delta, \delta]$ uniformly and generate its corresponding label as $y_i = F_{\theta^*}(x_i) + \epsilon$ where $\theta^* \in \mathbb{R}^D$ are the parameters of the ground-truth polynomial and $\epsilon$ is noise randomly drawn from a distribution we specify.
To create our test set, we use the same process but with $\epsilon = 0$.

We estimate the parameters of the predicted polynomial, $\hat{\theta} \in \mathbb{R}^{K}$, by minimizing the sum of the Huber loss over the training examples.
For many tasks, the exact model for the data is unknown, and the mismatch between the true and estimated model could be a source of noise in the predictions.
For this reason, we set $K > D$.

To evaluate the estimated parameters, we compute the root mean square error (RMSE) over the test set.
We use gradient descent to identity $\hat{\theta}$, and we use grid search to identify the optimal learning rate and $\alpha$ parameter.

For each experiment, we sample $N = 10000$ points with $\delta = 2$.
We arbitrarily select the parameters of our ground-truth polynomial as $\theta^* = [6, -3, -25, 15, 20, -10]$; therefore, $D=6$ and $F_{\theta^*}(x) = 6x^5 - 3x^4 - 25x^3 + 15x^2 + 20x - 10$.
Furthermore, we set $K=D+2$.
Since the Huber loss is designed to be robust to outliers, we sample $\epsilon$ from three different heavy-tailed distributions: the Laplace, Logistic, and Cauchy distributions.
The results of our experiments are shown in Figure \ref{fig:toy-problem}, which illustrates there is a near linear relationship between the scale of the label noise and the optimal $\alpha$ parameter for each of the distributions.
This suggests that knowing or estimating the label noise can enable us to identify a suitable $\alpha$ parameter.

Next, we will demonstrate with a real-world problem that approximating the label noise when it is unknown is an intuitive and effective method for selecting well-suited hyper-parameters.

\section{Case Study: Faster R-CNN}
\label{sec:case-study}

With our proposed interpretation, we analyze the loss functions used in a modern object detector, Faster R-CNN~\cite{faster-rcnn}, which is arguably one of the most important advancements in object detection in recent history.
Their work has inspired the development of several other object detectors including SSD~\cite{ssd}, FPN~\cite{fpn}, RetinaNet~\cite{retinanet}, and Mask R-CNN~\cite{mask-rcnn}, all of which leverage the same loss functions for bounding box regression.

The Faster R-CNN network architecture consists of two primary parts, a region proposal network and an object detection network.
The proposal network identifies regions that may contain objects, and the detection network refines and classifies the proposed regions.
To regress a bounding box, both the proposal network and the detection network utilize the Huber loss.
In their work, a bounding box is parameterized by its center and dimensions.
Let us start by analyzing the center prediction; the target for the $x$-coordinate of the center is
\begin{equation}
    t^*_x = \frac{x^* - x_a}{w_a}
\end{equation}
where $x^*$ is the $x$-coordinate of the ground-truth center, $x_a$ is the $x$-coordinate of the corresponding anchor, and $w_a$ is the width of the anchor.
A similar target is used for the center's $y$-coordinate except the height of the anchor is used instead of the width.
For the proposal network, the anchors are predefined, whereas the detection network uses the proposals as its anchors.

In the paper, the authors state that they use $\lambda H_1(t_x - t^*_x)$ to penalize the model's prediction, $t_x$, during training where $\lambda = 10$ is a weighting parameter \cite{faster-rcnn}.
To interpret this loss, let us first re-write the residual in terms of the center displacement,
\begin{align}
  \begin{split}
    t_x - t^*_x &= t_x - \frac{x^* - x_a}{w_a} \\
    &= \frac{\left(t_x w_a + x_a\right) - x^*}{w_a} = \frac{x - x^*}{w_a}
  \end{split}
\end{align}
where $x = t_x w_a + x_a$ is the predicted $x$-coordinate of the center.
Utilizing Equation \eqref{eqn:property}, we see that $\lambda H_1(t_x - t^*_x) \approx D_{w_a,\sfrac{w_a}{\lambda}}(x - x^*)$.
Based on this interpretation, the scale of noise in the prediction is one-tenth the width of the anchor, but the scale of the label uncertainty is the full width of the anchor.
Obviously, assuming the labels contain this amount of uncertainty is inappropriate.
As it happens, the loss function and targets used in the current implementation of Faster R-CNN differ significantly from the paper \cite{faster-rcnn-implementation}.
Interpreting the implementation is important because it is the foundation for several other object detectors \cite{mask-rcnn,fpn,retinanet,ssd}.

In the implementation of Faster R-CNN \cite{faster-rcnn-implementation}, the authors scale the Huber loss by $\sfrac{1}{\alpha}$.
Furthermore, the ground-truth targets have been shifted and scaled,
\begin{equation}
    \tilde{t}^*_x = \frac{t^*_x - \mu_x}{\sigma_x}
\end{equation}
by constant values $\mu_x \in \mathbb{R}$ and $\sigma_x \in \mathbb{R}^+$.
Let us repeat our analysis with these modifications.
Like before, we begin with re-writing the residual,
\begin{align}
  \begin{split}
    t_x - \tilde{t}^*_x &= t_x + \frac{\mu_x}{\sigma_x} - \frac{x^* - x_a}{\sigma_x w_a} \\
    &= \frac{\left[\left(t_x \sigma_x + \mu_x \right) w_a + x_a\right] - x^*}{\sigma_x w_a} = \frac{\tilde{x} - x^*}{\sigma_x w_a}
  \end{split}
\end{align}
where $\tilde{x} = \left(t_x \sigma_x + \mu_x \right) w_a + x_a$.
Next, let us consider the relationship between their loss function and our proposed loss function:
\begin{equation}
    \frac{\lambda}{\alpha} H_\alpha(t_x - \tilde{t}^*_x) \approx D_{\alpha \sigma_x w_a, \sfrac{\sigma_x w_a}{\lambda}}(\tilde{x} - x^*)\text{.}
\end{equation}
With these additional complexities, the authors were unknowingly able to independently manipulate the scale of the label and prediction noise.
To train the proposal network, $\lambda = 1$, $\alpha = \sfrac{1}{9}$, and $\sigma_x = 1$, and to train the detection network $\lambda=1$, $\alpha = 1$, and $\sigma_x = \sfrac{1}{10}$.
For both networks, the scale of the label noise is similar, a ninth and tenth of the anchor width, which is a much more reasonable assumption.
With this interpretation, the scale of prediction uncertainty is significantly larger for the proposal network compared to the detection network, the full width of the anchor versus a tenth of the width.
Intuitively, it makes sense to have a smaller prediction uncertainty for the detection network because it is designed to refine the output of the proposal network; however, a proposal uncertainty of this magnitude may be too extreme.

Likewise, we can perform the same analysis for the dimensions of the bounding box.
The target for the width of the bounding box is
\begin{equation}
    t^*_w = \log\frac{w^*}{w_a}
\end{equation}
and there is a similar target for the height of the bounding box.
As before, the target is shifted and scaled by $\mu_w \in \mathbb{R}$ and $\sigma_w \in \mathbb{R}^+$,
\begin{equation}
    \tilde{t}^*_w = \frac{t^*_w - \mu_w}{\sigma_w}\text{.}
\end{equation}
By re-writing the difference, we obtain the following:
\begin{align}
  \begin{split}
    t_w - \tilde{t}^*_w &= t_w - \frac{\log w^* - \log w_a - \mu_w}{\sigma_w} \\
    &= \frac{\left(t_w \sigma_w + \mu_w + \log w_a\right) - \log w^*}{\sigma_w} \\
    &= \frac{\log \tilde{w} - \log w^*}{\sigma_w}
  \end{split}
\end{align}
where $\tilde{w} = \exp\left(t_w \sigma_w + \mu_w \right) w_a$ is the predicted width of the bounding box.
Since the log of the width can be difficult to interpret, let us consider the following approximation:
\begin{equation}
    \log\frac{w^*}{w_a} \approx \frac{w^*}{w_a} - 1
\end{equation}
which is the first-order approximation of the logarithm when $\sfrac{w^*}{w_a} \approx 1$.
This is not an outlandish assumption because the intersection-over-union (IoU) between the anchor and the ground-truth bounding box needs to be significant for the ground-truth to be matched with the anchor.\footnote{Refer to Appendix~\ref{sec:target-approximation} for experimental validation of the target approximation.}
Now, the difference can be approximated as
\begin{align}
  \begin{split}
    t_w - \tilde{t}^*_w &\approx t_w + \frac{\mu_w + 1}{\sigma_w} - \frac{w^*}{\sigma_w w_a} \\
    &\approx \frac{\left(t_w \sigma_w + \mu_w + 1\right) w_a - w^*}{\sigma_w w_a} \approx \frac{\tilde{w} - w^*}{\sigma_w w_a}
  \end{split}
\end{align}
where $\tilde{w} \approx \left(t_w \sigma_w + \mu_w + 1\right) w_a$, which conforms with the first-order approximation of the exponential function when $t_w \sigma_w + \mu_w \approx 0$.
Leveraging our interpretation of the Huber loss, we observe
\begin{align}
  \begin{split}
    \frac{\lambda}{\alpha} H_\alpha(t_w - \tilde{t}^*_w) &\approx D_{\alpha \sigma_w, \sfrac{\sigma_w}{\lambda}}(\log \tilde{w} - \log w^*) \\
    &\approx D_{\alpha \sigma_w w_a, \sfrac{\sigma_w w_a}{\lambda}}(\tilde{w} - w^*)\text{.}
  \end{split}
\end{align}
In this case, $\lambda = 1$, $\alpha = \sfrac{1}{9}$, and $\sigma_w = 1$ for the proposal network, and $\lambda = 1$, $\alpha = 1$, and $\sigma_w = \sfrac{1}{5}$ for the detection network.
Interestingly, the label noise is assumed to be higher for the detection network compared to the proposal network, which could be less than optimal.

It is unclear how the authors arrived at these peculiar hyper-parameters, undoubtedly through some form of parameter sweep.
Based on our interpretation, we believe the hyper-parameters could be improved upon, which we demonstrate in the following section.
In general, we believe that our interpretation can aid in hyper-parameter selection by eliminating inappropriate values.

\section{Experiments}
\label{sec:experiments}

\begin{table*}[t]
  \centering
  \caption{List of Hyper-Parameters}
  \vspace{-0.5em}
  \scalebox{0.88}{
    \begin{tabular}{c|cc|cc|cc|cc|cc}
      \multirow{2}{*}{Parameters} & \multicolumn{2}{c|}{Publication} & \multicolumn{2}{c|}{Implementation} & \multicolumn{2}{c|}{Experiment A} & \multicolumn{2}{c|}{Experiment B} & \multicolumn{2}{c}{Experiment C} \\
      & Proposal & Detection & Proposal & Detection & Proposal & Detection & Proposal & Detection & Proposal & Detection \\
      \hline
      \hline
      $\lambda$ & $10$ & $10$ & $1$ & $1$ & $\sfrac{1}{4}$ & $\sfrac{1}{2}$ & $\sfrac{1}{2}$ & $1$ & $\sfrac{1}{4}$ & $1$ \\
      $\alpha$ & $1$ & $1$ & $\sfrac{1}{9}$ & $1$ & $1$ & $1$ & $1$ & $1$ & $1$ & $1$ \\
      $\sigma_x$ & $1$ & $1$ & $1$ & $\sfrac{1}{10}$ & $\sfrac{1}{20}$ & $\sfrac{1}{20}$ & $\sfrac{1}{20}$ & $\sfrac{1}{20}$ & $\sfrac{1}{20}$ & $\sfrac{1}{20}$ \\
      $\sigma_y$ & $1$ & $1$ & $1$ & $\sfrac{1}{10}$ & $\sfrac{1}{20}$ & $\sfrac{1}{20}$ & $\sfrac{1}{20}$ & $\sfrac{1}{20}$ & $\sfrac{1}{20}$ & $\sfrac{1}{20}$ \\
      $\sigma_w$ & $1$ & $1$ & $1$ & $\sfrac{1}{5}$ & $\sfrac{1}{10}$ & $\sfrac{1}{10}$ & $\sfrac{1}{10}$ & $\sfrac{1}{10}$ & $\sfrac{1}{10}$ & $\sfrac{1}{10}$ \\
      $\sigma_h$ & $1$ & $1$ & $1$ & $\sfrac{1}{5}$ & $\sfrac{1}{10}$ & $\sfrac{1}{10}$ & $\sfrac{1}{10}$ & $\sfrac{1}{10}$ & $\sfrac{1}{10}$ & $\sfrac{1}{10}$ \\
    \end{tabular}
    \label{tab:hyper-parameters}
  }
  \vspace{1em}
  \caption{Interpreted Scale of the Label and Prediction Uncertainties}
  \vspace{-0.5em}
  \scalebox{0.88}{
    \begin{tabular}{c|cc|cc|cc|cc|cc}
      \multirow{2}{*}{Bounding Box} & \multicolumn{2}{c|}{Publication} & \multicolumn{2}{c|}{Implementation} & \multicolumn{2}{c|}{Experiment A} & \multicolumn{2}{c|}{Experiment B} & \multicolumn{2}{c}{Experiment C} \\
      & Proposal & Detection & Proposal & Detection & Proposal & Detection & Proposal & Detection & Proposal & Detection \\
      \hline
      \hline
      $x^*$ & $w_a$ & $w_a$ & $\sfrac{w_a}{9}$ & $\sfrac{w_a}{10}$ & $\sfrac{w_a}{20}$ & $\sfrac{w_a}{20}$ & $\sfrac{w_a}{20}$ & $\sfrac{w_a}{20}$ & $\sfrac{w_a}{20}$ & $\sfrac{w_a}{20}$ \\
      $y^*$ & $h_a$ & $h_a$ & $\sfrac{h_a}{9}$ & $\sfrac{h_a}{10}$ & $\sfrac{h_a}{20}$ & $\sfrac{h_a}{20}$ & $\sfrac{h_a}{20}$ & $\sfrac{h_a}{20}$ & $\sfrac{h_a}{20}$ & $\sfrac{h_a}{20}$ \\
      $w^*$ & $w_a$ & $w_a$ & $\sfrac{w_a}{9}$ & $\sfrac{w_a}{5}$ & $\sfrac{w_a}{10}$ & $\sfrac{w_a}{10}$ & $\sfrac{w_a}{10}$ & $\sfrac{w_a}{10}$ & $\sfrac{w_a}{10}$ & $\sfrac{w_a}{10}$ \\
      $h^*$ & $h_a$ & $h_a$ & $\sfrac{h_a}{9}$ & $\sfrac{h_a}{5}$ & $\sfrac{h_a}{10}$ & $\sfrac{h_a}{10}$ & $\sfrac{h_a}{10}$ & $\sfrac{h_a}{10}$ & $\sfrac{h_a}{10}$ & $\sfrac{h_a}{10}$ \\
      \hline
      $\tilde{x}$ & $\sfrac{w_a}{10}$ & $\sfrac{w_a}{10}$ & $w_a$ & $\sfrac{w_a}{10}$ & $\sfrac{w_a}{5}$ & $\sfrac{w_a}{10}$ & $\sfrac{w_a}{10}$ & $\sfrac{w_a}{20}$ & $\sfrac{w_a}{5}$ & $\sfrac{w_a}{20}$ \\
      $\tilde{y}$ & $\sfrac{h_a}{10}$ & $\sfrac{h_a}{10}$ & $h_a$ & $\sfrac{h_a}{10}$ & $\sfrac{h_a}{5}$ & $\sfrac{h_a}{10}$ & $\sfrac{h_a}{10}$ & $\sfrac{h_a}{20}$ & $\sfrac{h_a}{5}$ & $\sfrac{h_a}{20}$ \\
      $\tilde{w}$ & $\sfrac{w_a}{10}$ & $\sfrac{w_a}{10}$ & $w_a$ & $\sfrac{w_a}{5}$ & $\sfrac{2w_a}{5}$ & $\sfrac{w_a}{5}$ & $\sfrac{w_a}{5}$ & $\sfrac{w_a}{10}$ & $\sfrac{2w_a}{5}$ & $\sfrac{w_a}{10}$ \\
      $\tilde{h}$ & $\sfrac{h_a}{10}$ & $\sfrac{h_a}{10}$ & $h_a$ & $\sfrac{h_a}{5}$ & $\sfrac{2h_a}{5}$ & $\sfrac{h_a}{5}$ & $\sfrac{h_a}{5}$ & $\sfrac{h_a}{10}$ & $\sfrac{2h_a}{5}$ & $\sfrac{h_a}{10}$ \\
    \end{tabular}
    \label{tab:uncertainties}
  }
\end{table*}

\begin{table}[t]
  \centering
  \caption{Faster R-CNN Performance}
  \vspace{-0.5em}
  \scalebox{0.88}{
    \begin{tabular}{c|ccc}
      \multirow{2}{*}{Parameters} & \multicolumn{3}{c}{Mean Average Precision (mAP) @} \\
      & 0.5 IoU & 0.75 IoU & 0.5-0.95 IoU \\
      \hline
      \hline
      Baseline \cite{faster-rcnn} & 41.5 & - & 21.2 \\
      \hline
      Publication & 42.8 & 18.7 & 21.0 \\
      Implementation & \textbf{44.7} & 23.1 & 23.8 \\
      \hline
      Experiment A & \textbf{44.7} & 24.0 & 24.2 \\
      Experiment B & 44.2 & \textbf{25.0} & 24.6 \\
      Experiment C & 44.6 & 24.9 & \textbf{24.7} \\
    \end{tabular}
    \label{tab:faster-results}
  }
\end{table}

\begin{table}[t]
  \centering
  \caption{RetinaNet Performance}
  \vspace{-0.5em}
  \scalebox{0.88}{
    \begin{tabular}{c|ccc}
      \multirow{2}{*}{Parameters} & \multicolumn{3}{c}{Mean Average Precision (mAP) @} \\
      & 0.5 IoU & 0.75 IoU & 0.5-0.95 IoU \\
      \hline
      \hline
      Implementation & 60.1 & 42.9 & 40.4 \\
      \hline
      Experiment A & \textbf{60.6} & \textbf{43.5} & \textbf{40.8} \\
      Experiment B & 58.9 & 42.3 & 39.8 \\
    \end{tabular}
    \label{tab:retinanet-results}
  }
\end{table}

In this section, we perform experiments on Faster R-CNN as well as another modern object detector, RetinaNet.
Our goal is not to obtain state-of-the-art object detection performance, there is a wealth of literature that improves upon these methods; instead, our goal is to demonstrate that our proposed interpretation of the Huber loss can lead to hyper-parameters better suited to the task of bounding box regression.
Furthermore, our aim is not to replace the Huber loss with our proposed loss; rather, we want to leverage the relationship between the losses to gain insight into the Huber loss.\footnote{For completion, we demonstrate that replacing the Huber loss with our proposed loss function produces comparable results in Appendix~\ref{sec:loss-experiment}.}
For these reasons, we limit our modifications to the following hyper-parameters: $\alpha$, $\lambda$, $\sigma_x$, $\sigma_y$, $\sigma_w$, and $\sigma_h$ (refer to Section \ref{sec:case-study} for more details).\footnote{The hyper-parameters $\mu_x$, $\mu_y$, $\mu_w$, and $\mu_h$ are all set to zero in the implementation, and they are left unchanged in all the experiments.}

\subsection{Faster R-CNN}
\label{sec:experiments-faster}

To conduct our experiments, we utilize the implementation and framework provided by the authors of Faster R-CNN \cite{faster-rcnn-implementation}.
The deepest neural network supported by their framework is VGG-16 \cite{vgg16}, and the largest dataset is MS-COCO 2014 \cite{coco}.
For all of our experiments, we train the Faster R-CNN model with a VGG-16 backbone on the MS-COCO 2014 training set and measure the object detection performance on the validation set.
The MS-COCO 2014 dataset \cite{coco} contains objects from 80 different classes, and it includes over 80k images for training and 40k images for validation.
The metric used to measure object detection performance is the mean average precision (mAP) at various intersection-over-union (IoU) thresholds.
To evaluate our experiments, we consider the mAP at 0.5 IoU and 0.75 IoU thresholds, as well as, the mAP averaged over 0.5-0.95 IoU thresholds.
Unless otherwise stated, we use the default configurations set by the authors to train and test the models.

For our initial experiment, we train a model using the hyper-parameters as they are described in the publication~\cite{faster-rcnn}.
Afterwards, we evaluate the parameters as they are specified in the current implementation of Faster R-CNN~\cite{faster-rcnn-implementation}.
Lastly, we leverage our interpretation to propose three new sets of hyper-parameters.
A full list of the parameters used as well as their corresponding interpretation is provided in Table~\ref{tab:hyper-parameters} and Table~\ref{tab:uncertainties}, respectively.
Notice for our proposed hyper-parameters, the label noise does not vary between the proposal and detection network, only the prediction noise varies.
Based on our interpretation, we believe the label noise should not change between the two networks while the prediction noise should.

The results of the experiments are presented in Table \ref{tab:faster-results}.
We were unable to exactly reproduce the results as they are listed in \cite{faster-rcnn}, likely due to changes made to the implementation by the authors that are unrelated to the hyper-parameters of the Huber loss.
Regardless, in our experiments, the published hyper-parameters perform the worst by a significant margin, which should not be a surprise given our interpretation.
The authors of Faster R-CNN were able to improve performance of the detector by tuning the hyper-parameters in the implementation \cite{faster-rcnn-implementation}.
We were able to further improve performance by reducing the estimated amount of noise in the labels and predictions.
Specifically, we were able to raise performance at larger IoU thresholds.
Achieving an improvement in mAP at higher thresholds requires more accurate bounding boxes; therefore, it makes sense that reducing the estimated uncertainty increases performance at those thresholds.
Experiment A and B trade-off performance at 0.5 and 0.75 IoU, and Experiment C identifies a good balance between both.
These results are significant because they were obtained by leveraging the intuition provided by our proposed interpretation of the Huber loss without the need for an exhaustive hyper-parameter search.

\subsection{RetinaNet}
\label{sec:experiments-retinanet}

As previously discussed, Faster R-CNN uses two networks or stages to perform object detection.
Whereas, RetinaNet \cite{retinanet} uses only a single stage; therefore, it uses one network to regress the bounding box and classify the objects.
For our experiments, we utilize the official implementation of RetinaNet \cite{detectron2}.
In the RetinaNet implementation, the loss function utilized to regress bounding boxes is identical to the loss function used for the proposal network in the Faster R-CNN implementation \cite{faster-rcnn-implementation}.
For this reason, we repeat Experiments A and B from Section \ref{sec:experiments-faster} with the RetinaNet model.\footnote{Experiment C is not repeated since A and C use the same parameters for the proposal network.}

Since RetinaNet is a more recently proposed detector, the implementation supports more sophisticated backbone networks and newer datasets.
For all of the experiments, we train the RetinaNet model with a ResNet-101 \cite{resnet} backbone on the MS-COCO 2017 \cite{coco} training set and measure the object detection performance on the validation set utilizing the same metrics as Section \ref{sec:experiments-faster}.

The results of the experiments are presented in Table~\ref{tab:retinanet-results}.
Experiment A was able to achieve higher performance across the board, while Experiment B degraded performance significantly at 0.5 IoU.
We observed a similar trend in Section \ref{sec:experiments-faster}.
These results demonstrate that our proposed interpretation can identify well-suited hyper-parameters for a task regardless of the underlying meta-architecture, backbone network, and dataset.

\section{Conclusion}
\label{sec:conclusion}

In this work, we propose an alternative probabilistic interpretation of the Huber loss.
Our interpretation connects the Huber loss to the KL divergence of Laplace distributions, which provides an intuitive understanding of its parameters.
We demonstrated that our interpretation can aid in hyper-parameter selection, and we were able to improve the performance of the Faster R-CNN and RetinaNet object detectors without needing to exhaustively search over hyper-parameters.

The vast majority of recent papers that utilize the Huber loss \cite{chang2018,chen2016,mv3d,alp2017,mask-rcnn,avod,ku2019,fpn,ssd,pixor}, use the formulation as described in the Fast or Faster R-CNN publications \cite{fast-rcnn,faster-rcnn}.
Therefore, these methods as well as future methods have the potential to be improved significantly by leveraging our proposed interpretation of the Huber loss to identify better suited hyper-parameters for their respective tasks.

{\small
\bibliographystyle{ieee_fullname}
\bibliography{egbib}
}

\newpage
\appendix
\section*{Appendix}
\section{Derivation of the Kullback-Leibler\\Divergence of Laplace Distributions}
\label{sec:derivation}

The Kullback-Leibler (KL) divergence between a probability distribution $q(x)$ and a reference distribution $p(x)$ is defined as follows:
\begin{align}
  \begin{split}
    &D(p(x)\|q(x)) \\ 
    &\quad = H(p(x),q(x)) - H(p(x)) \\
    &\quad = - \int_{-\infty}^\infty p(x) \log q(x) dx + \int_{-\infty}^\infty p(x) \log p(x) dx
  \end{split}
\end{align}
where $H(p(x))$ is the entropy of $p(x)$ and $H(p(x),q(x))$ is the cross entropy between $p(x)$ and $q(x)$.
When both $p(x)$ and $q(x)$ are Laplace distributions,
\begin{equation}
    p(x) = \frac{1}{2b_1} \exp\left(-\frac{\left|x - \mu_1\right|}{b_1}\right)
\end{equation}
and
\begin{equation}
    q(x) = \frac{1}{2b_2} \exp\left(-\frac{\left|x - \mu_2\right|}{b_2}\right)\text{,}
\end{equation}
the cross entropy becomes
\begin{align}
  \begin{split}
    &H(p(x),q(x)) \\
    &\quad = -\int_{-\infty}^\infty p(x) \log q(x) dx \\
    &\quad = \int_{-\infty}^\infty \frac{\left|x - \mu_2\right|}{2b_1 b_2} \exp\left(-\frac{\left|x - \mu_1\right|}{b_1}\right) dx + \log(2b_2)\text{.}
  \end{split}
\end{align}
To evaluate the integral, consider the case when $\mu_1 \geq \mu_2$,
\begin{align}
  \begin{split}
    &\int_{-\infty}^\infty \frac{\left|x - \mu_2\right|}{2b_1 b_2} \exp\left(-\frac{\left|x - \mu_1\right|}{b_1}\right) dx \\
    &\quad = \int_{-\infty}^{\mu_2} \frac{\mu_2 - x}{2b_1 b_2} \exp\left(-\frac{\mu_1 - x}{b_1}\right) dx \\
    &\quad + \int_{\mu_2}^{\mu_1} \frac{x - \mu_2}{2b_1 b_2} \exp\left(-\frac{\mu_1 - x}{b_1}\right) dx \\
    &\quad + \int_{\mu_1}^{\infty} \frac{x - \mu_2}{2b_1 b_2} \exp\left(-\frac{x - \mu_1}{b_1}\right) dx
  \end{split}
\end{align}
and when $\mu_1 < \mu_2$,
\begin{align}
  \begin{split}
    &\int_{-\infty}^\infty \frac{\left|x - \mu_2\right|}{2b_1 b_2} \exp\left(-\frac{\left|x - \mu_1\right|}{b_1}\right) dx \\
    &\quad = \int_{-\infty}^{\mu_1} \frac{\mu_2 - x}{2b_1 b_2} \exp\left(-\frac{\mu_1 - x}{b_1}\right) dx \\
    &\quad + \int_{\mu_1}^{\mu_2} \frac{\mu_2 - x}{2b_1 b_2} \exp\left(-\frac{x - \mu_1}{b_1}\right) dx \\
    &\quad+ \int_{\mu_2}^{\infty} \frac{x - \mu_2}{2b_1 b_2} \exp\left(-\frac{x - \mu_1}{b_1}\right) dx\text{.}
  \end{split}
\end{align}
Evaluating each of the integrals produces the following result:
\begin{align}
  \begin{split}
    &\int_{-\infty}^\infty \frac{\left|x - \mu_2\right|}{2b_1 b_2} \exp\left(-\frac{\left|x - \mu_1\right|}{b_1}\right) dx \\
    &\quad =
    \begin{cases} 
        \frac{b_1 \exp \left(-\frac{\mu_1 - \mu_2}{b_1}\right) + (\mu_1 - \mu_2)}{b_2}, & \mu_1 \geq \mu_2 \\
        \frac{b_1 \exp \left(-\frac{\mu_2 - \mu_1}{b_1}\right) + (\mu_2 - \mu_1)}{b_2}, & \mu_1 < \mu_2\text{.}
    \end{cases}
  \end{split}
\end{align}
Altogether, the cross entropy between two Laplace distribution is
\begin{align}
  \begin{split}
    &H(p(x),q(x)) \\
    &\quad = -\int_{-\infty}^\infty p(x) \log q(x) dx \\
    &\quad = \frac{b_1 \exp\left(-\frac{\left|\mu_1 - \mu_2\right|}{b_1} \right) + \left|\mu_1 - \mu_2\right|}{b_2} + \log(2b_2)
  \end{split}
\end{align}
and the entropy of a Laplace distribution is
\begin{equation}
    H(p(x)) = -\int_{-\infty}^\infty p(x) \log p(x) dx = 1 + \log(2b_1)\text{.}
\end{equation}
As a result, the KL divergence between two Laplace distributions is
\begin{align}
  \begin{split}
    &D(p(x)\|q(x)) \\
    &\quad = \frac{b_1 \exp\left(-\frac{\left|\mu_1 - \mu_2\right|}{b_1} \right) + \left|\mu_1 - \mu_2\right|}{b_2} + \log\frac{b_2}{b_1} - 1\text{.}
  \end{split}
\end{align}

\section{Proof of Differentiability}
\label{sec:derivatives}

In Section~\ref{sec:relationship}, we utilize a second-order approximation of our proposed loss function (Equation \eqref{eqn:kl_loss}) about zero to illustrate its relationship with the Huber loss (Equation \eqref{eqn:huber_loss}).
In this section, we will derive the first and second derivatives of our loss and prove their existence at zero.

Equation \eqref{eqn:kl_loss} can be written as follows:
\begin{equation}
    D_{\alpha,\beta}(x) =
    \begin{cases}
        \frac{\alpha\exp\left(-\frac{x}{\alpha}\right) + x - \alpha}{\beta} & x \geq 0 \\
        \frac{\alpha\exp\left(\frac{x}{\alpha}\right) - x - \alpha}{\beta} & x < 0\text{.}
    \end{cases}
\end{equation}
Therefore, its first derivative is
\begin{align}
  \begin{split}
    D'_{\alpha,\beta}(x) &=
    \begin{cases}
        \frac{1 - \exp\left(-\frac{x}{\alpha}\right)}{\beta} & x \geq 0 \\
        -\frac{1 - \exp\left(\frac{x}{\alpha}\right)}{\beta} & x < 0
    \end{cases} \\
    &= \frac{\sgn(x)}{\beta} \left(1 - \exp\left( -\frac{|x|}{\alpha} \right) \right)\text{,}
  \end{split}
\end{align}
and its second derivative is
\begin{equation}
    D''_{\alpha,\beta}(x) =
    \begin{cases}
        \frac{\exp\left(-\frac{x}{\alpha}\right)}{\alpha\beta} & x \geq 0 \\
        \frac{ \exp\left(\frac{x}{\alpha}\right)}{\alpha\beta} & x < 0
    \end{cases}
     = \frac{1}{\alpha \beta} \exp\left(-\frac{|x|}{\alpha} \right)\text{.}
\end{equation}
To prove the existence of the derivatives at $x=0$, we need to show that both $D_{\alpha,\beta}(x)$ and $D'_{\alpha,\beta}(x)$ are differentiable at $x=0$.
The derivative of any function $f(x)$ at $x=a$ is defined as follows:
\begin{equation}
    f'(a) = \lim_{x \rightarrow a} \frac{f(x)-f(a)}{x-a}\text{.}
\end{equation}
The function $f(x)$ is said to be differentiable at $x=a$ when the limit exists.
To prove our claims, we will rely heavily on the following well-known identity:
\begin{equation}
     \lim_{x \rightarrow 0} (1+x)^{\frac{1}{x}} = e\text{.}
     \label{eqn:identity}
\end{equation}

\begin{claim}
    The following function is differentiable at ${x=0}$:
    \begin{equation}
        D_{\alpha,\beta}(x) = \frac{\alpha\exp\left(-\frac{|x|}{\alpha}\right) + |x| - \alpha}{\beta}
    \end{equation}
    when $\alpha > 0$ and $\beta > 0$.
\end{claim}

\begin{proof}
    The derivative of $D_{\alpha,\beta}(x)$ at $x=0$ is defined as
    \begin{align}
      \begin{split}
        D'_{\alpha,\beta}(0) &= \lim_{x \rightarrow 0} \frac{\alpha\exp\left(-\frac{|x|}{\alpha}\right)+ |x| - \alpha }{\beta x} \\
        &= \lim_{x \rightarrow 0} \frac{\alpha\left(\exp\left(-\frac{|x|}{\alpha}\right) - 1\right)}{\beta x} + \lim_{x \rightarrow 0} \frac{|x|}{\beta x}\text{.}
      \end{split}
    \end{align}
    The right limit of the first term is equal to the following:
    \begin{equation}
        \lim_{x \rightarrow 0^+} \frac{\alpha\left(\exp\left(-\frac{x}{\alpha}\right) - 1\right)}{\beta x} =  \lim_{u \rightarrow 0^+} -\frac{u}{\beta \log(u+1)}
    \end{equation}
    where we substitute $u=\exp(-\sfrac{x}{\alpha})-1$ in for $x$; therefore, $x=-\alpha\log(u+1)$ and $u \rightarrow 0$ as $x \rightarrow 0$.
    Utilizing Equation~\eqref{eqn:identity}, the limit becomes
    \begin{align}
      \begin{split}
        \lim_{u \rightarrow 0^+} -\frac{u}{\beta \log(u+1)} &= \lim_{u \rightarrow 0^+} -\frac{1}{\beta \log(u+1)^\frac{1}{u}} \\
        &=  -\frac{1}{\beta \log\left(\lim_{u \rightarrow 0^+}(u+1)^\frac{1}{u}\right)} \\
        &= -\frac{1}{\beta}\text{.}
      \end{split}
    \end{align}
    Similarly, for the left limit of the first term,
    \begin{equation}
        \lim_{x \rightarrow 0^-} \frac{\alpha\left(\exp\left(\frac{x}{\alpha}\right) - 1\right)}{\beta x} =\lim_{v \rightarrow 0^-} \frac{v}{\beta \log(v+1)} = \frac{1}{\beta}
    \end{equation}
    where we substitute $v=\exp(\sfrac{x}{\alpha})-1$ in for $x$; as a result, $x=\alpha\log(v+1)$ and $v \rightarrow 0$ as $x \rightarrow 0$.
    Furthermore, the right limit of the second term is
    \begin{equation}
        \lim_{x \rightarrow 0^+} \frac{|x|}{\beta x} = \lim_{x \rightarrow 0^+} \frac{x}{\beta x} = \frac{1}{\beta}\text{,}
    \end{equation}
    and the left limit is
    \begin{equation}
        \lim_{x \rightarrow 0^-} \frac{|x|}{\beta x} = \lim_{x \rightarrow 0^-} -\frac{x}{\beta x} = -\frac{1}{\beta}\text{.}
    \end{equation}
    By adding both terms together, both sides of the limit become zero, which means the limit exists and proves $D_{\alpha,\beta}(x)$ is differentiable at ${x=0}$.
\end{proof}

\begin{claim}
    The following function is differentiable at ${x=0}$:
    \begin{equation}
        D'_{\alpha,\beta}(x) = \frac{\sgn(x)}{\beta} \left( 1 - \exp\left( -\frac{|x|}{\alpha} \right) \right)
    \end{equation}
    when $\alpha > 0$ and $\beta > 0$.
\end{claim}

\begin{proof}
    The derivative of $D'_{\alpha,\beta}(x)$ at $x=0$ is defined as
    \begin{equation}
        D''_{\alpha,\beta}(0) = \lim_{x \rightarrow 0} \frac{\sgn(x)\left(1 - \exp\left( -\frac{|x|}{\alpha} \right) \right)}{\beta x}\text{.}
    \end{equation}
    The right limit is equal to the following:
    \begin{equation}
        \lim_{x \rightarrow 0^+} \frac{1 - \exp\left( -\frac{x}{\alpha}\right)}{\beta x} = \lim_{u \rightarrow 0^+} \frac{u}{\alpha\beta \log(u+1)}
    \end{equation}
    where we substitute $u=\exp(-\sfrac{x}{\alpha})-1$ in for $x$; therefore, $x=-\alpha\log(u+1)$ and $u \rightarrow 0$ as $x \rightarrow 0$.
    Utilizing Equation~\eqref{eqn:identity}, the limit becomes
    \begin{align}
      \begin{split}
        \lim_{u \rightarrow 0^+} \frac{u}{\alpha\beta \log(u+1)} &= \lim_{u \rightarrow 0^+} \frac{1}{\alpha\beta \log(u+1)^\frac{1}{u}} \\
        &= \frac{1}{\alpha\beta \log\left(\lim_{u \rightarrow 0^+} (u+1)^\frac{1}{u}\right)} \\
        &= \frac{1}{\alpha\beta}
      \end{split}
    \end{align}
    Similarly, for the left limit,
    \begin{equation}
        \lim_{x \rightarrow 0^-} -\frac{1 - \exp\left(\frac{x}{\alpha}\right)}{\beta x} = \lim_{v \rightarrow 0^-} \frac{v}{\alpha\beta \log(v+1)} = \frac{1}{\alpha\beta}
    \end{equation}
    where we substitute $v=\exp(\sfrac{x}{\alpha})-1$ in for $x$; as a result, $x=\alpha\log(v+1)$ and $v \rightarrow 0$ as $x \rightarrow 0$.
    Therefore, both sides of the limit are equal, which means the limit exists and proves $D'_{\alpha,\beta}(x)$ is differentiable at ${x=0}$.
\end{proof}

\section{Proof of Inequalities}
\label{sec:proofs}

In Section~\ref{sec:relationship}, we state that the Huber loss, $H_\alpha(x)$, is bounded below by $ D_{\alpha,\sfrac{1}{\alpha}}(x)$ and above by $D_{\sfrac{\alpha}{2},\sfrac{1}{\alpha}}(x)$, and the bounds are tight.
In this section, we prove these claims.
Since the loss functions are symmetric about $x = 0$, it is sufficient to prove only when $x \geq 0$.

\begin{claim}
    The following inequality holds for all ${x \in \mathbb{R}}$:
    \begin{equation}
        H_\alpha(x) - D_{\alpha,\sfrac{1}{\alpha}}(x) \geq 0
    \end{equation}
\end{claim}

\begin{proof}
    When $0 \leq x \leq \alpha$, the inequality is
    \begin{equation}
        \frac{1}{2} x^2 - \alpha x + \alpha^2 - \alpha^2 \exp\left(-\frac{x}{\alpha}\right) \geq 0
    \end{equation}
    and it becomes
    \begin{equation}
       \frac{1}{2} \alpha^2  - \alpha^2  \exp\left(-\frac{x}{\alpha}\right) \geq 0
    \end{equation}
    when $x \geq \alpha$.
    The inequalities can be simplified by substituting $y = \sfrac{x}{\alpha}$ and dividing by $\alpha^2$.
    As a result, we now need to prove
    \begin{equation}
        f_1(y) = \frac{1}{2} y^2 - y + 1 - \exp(-y) \geq 0
        \label{eqn:inequality-1}
    \end{equation}
    when $0 \leq y \leq 1$, and
    \begin{equation}
        f_2(y) = \frac{1}{2} - \exp(-y) \geq 0
        \label{eqn:inequality-2}
    \end{equation}
    when $y \geq 1$.
    Equation \eqref{eqn:inequality-1} is a well-known inequality and can be proven by utilizing the mean value theorem.
    The first and second derivative of $f_1(y)$ are
    \begin{equation}
        f'_1(y) = y - 1 + \exp(-y)
    \end{equation}
    and
    \begin{equation}
        f''_1(y) = 1 - \exp(-y)\text{.}
    \end{equation}
    When $y \geq 0$, $f'_1(y) \geq 0$ since $f''_1(y) \geq 0$ and ${f'_1(0) = 0}$; likewise, $f_1(y) \geq 0$ for the same reason, $f'_1(y) \geq 0$ and $f_1(0) = 0$.
    The proof of the second inequality follows directly from the first.
    From Equation \eqref{eqn:inequality-1}, we know that ${\exp(-1) \leq \sfrac{1}{2}}$; therefore, $f_2(y) \geq 0$ when $y \geq 1$ since $\exp(-y)$ is monotonically decreasing.
\end{proof}

\begin{claim}
    The following inequality holds for all ${x \in \mathbb{R}}$:
    \begin{equation}
        D_{\sfrac{\alpha}{2},\sfrac{1}{\alpha}}(x) - H_\alpha(x) \geq 0
    \end{equation}
\end{claim}

\begin{proof}
    The inequality is equal to
    \begin{equation}
        \frac{\alpha^2}{2} \exp\left(-\frac{2}{\alpha}x\right) + \alpha x - \frac{\alpha^2}{2} - \frac{1}{2} x^2 \geq 0
    \end{equation}
    when $0 \leq x \leq \alpha$, and it is
    \begin{equation}
       \frac{\alpha^2}{2} \exp\left(-\frac{2}{\alpha}x\right) \geq 0
    \end{equation}
    when $x \geq \alpha$.
    Again, the inequalities can be simplified by substituting $y = \sfrac{2x}{\alpha}$ and dividing by $\sfrac{\alpha^2}{2}$, which results in the following inequalities:
    \begin{equation}
        f_3(y) = \exp(-y) + y - 1 - \frac{1}{4} y^2 \geq 0
        \label{eqn:inequality-3}
    \end{equation}
    when $0 \leq y \leq 2$, and
    \begin{equation}
        f_4(y) = \exp(-y) \geq 0
        \label{eqn:inequality-4}
    \end{equation}
    when $y \geq 2$.
    The second inequality, $f_4(y) \geq 0$, clearly holds for all $y \in \mathbb{R}$; whereas, the first inequality, $f_3(y) \geq 0$, is less obvious.
    The first and second derivative of $f_3(y)$ are
    \begin{equation}
        f'_3(y) = -\exp(-y) + 1 - \frac{1}{2} y
    \end{equation}
    and
    \begin{equation}
        f''_3(y) = \exp(-y) - \frac{1}{2}\text{.}
    \end{equation}
    At $y = 0$, $f'_3(0) = 0$ and $f''_3(0) = \sfrac{1}{2} > 0$, and at $y = 2$, $f'_3(2) = -\exp(-2) < 0$.
    Since $f''_3(y)$ has a single root, $f'_3(y)$ can have at most two roots by Rolle's theorem.
    Therefore, there exists a unique value, $0 < y_0 < 2$, where $f'_3(y_0) = 0$, and on the interval $0 \leq y \leq y_0$, $f'_3(y) \geq 0$.
    Moreover, by the mean value theorem, $f_3(y) \geq 0$ on that interval, $0 \leq y \leq y_0$, since $f'_3(y) \geq 0$ and $f_3(0) = 0$.
    Note that $f'_3(y) \leq 0$, or equivalently
    \begin{equation}
        \exp(-y) \geq 1 - \frac{1}{2} y
    \end{equation}
    on the interval $y_0 \leq y \leq 2$.
    Consequently, to complete the proof of $f_3(y) \geq 0$, we just need to show that
    \begin{equation}
        1 - \frac{1}{2} y \geq 1 - y + \frac{1}{4} y^2
    \end{equation}
    or correspondingly
    \begin{equation}
        f_5(y) = -\frac{1}{4} y^2 + \frac{1}{2} y \geq 0
        \label{eqn:inequality-5}
    \end{equation}
    when $y_0 \leq y \leq 2$. 
    The roots of $f_5(y)$ are at $y = 0$ and $y = 2$, since $f_5(1) = \sfrac{1}{4} > 0$, $f_5(y) \geq 0$ on the interval $0 \leq y \leq 2$.
\end{proof}

\begin{claim}
    For all $x \in \mathbb{R}$, $H_\alpha(x)$ is tightly bounded between $D_{\alpha,\sfrac{1}{\alpha}}(x)$ and $D_{\sfrac{\alpha}{2},\sfrac{1}{\alpha}}(x)$.
    Therefore, the inequalities
    \begin{equation}
        D_{\alpha,\sfrac{1}{\alpha}}(x) \leq D_{\alpha_1,\beta_1}(x) \leq H_\alpha(x)
    \end{equation}
    and
    \begin{equation}
        H_\alpha(x) \leq D_{\alpha_2,\beta_2}(x) \leq  D_{\sfrac{\alpha}{2},\sfrac{1}{\alpha}}(x)
    \end{equation}
    hold only, for all $x \in \mathbb{R}$, when $\alpha_1 = \alpha$, $\alpha_2 = \sfrac{\alpha}{2}$, and $\beta_1 = \beta_2 = \sfrac{1}{\alpha}$.
\end{claim}

\begin{proof}
    The inequalities are equivalent to
    \begin{equation}
        D_{\alpha,\sfrac{1}{\alpha}}(x) - H_\alpha(x) \leq D_{\alpha_1,\beta_1}(x) - H_\alpha(x) \leq 0
    \end{equation}
    and
    \begin{equation}
        0 \leq D_{\alpha_2,\beta_2}(x) - H_\alpha(x)\leq  D_{\sfrac{\alpha}{2},\sfrac{1}{\alpha}}(x) - H_\alpha(x)\text{.}
    \end{equation}
    As $x$ goes to infinity,
    \begin{align}
        \lim_{x \rightarrow \infty} &D_{\sfrac{\alpha}{2},\sfrac{1}{\alpha}}(x) - H_\alpha(x) = 0 \\
        \lim_{x \rightarrow \infty} &D_{\alpha,\sfrac{1}{\alpha}}(x) - H_\alpha(x) = -\frac{1}{2}\alpha^2
    \end{align}
    and
    \begin{equation}
        \lim_{x \rightarrow \infty} D_{\alpha_*,\beta_*}(x) - H_\alpha(x) =
        \begin{cases}
            \infty, & \beta_* < \frac{1}{\alpha} \\
            \alpha \left(\frac{1}{2}\alpha - \alpha_* \right), & \beta_* = \frac{1}{\alpha} \\
            -\infty, & \beta_* > \frac{1}{\alpha}\text{.}
        \end{cases}
    \end{equation}
    For the inequalities to hold in the limit, $\beta_*$ must equal $\sfrac{1}{\alpha}$ regardless of the value of $\alpha_*$, $\alpha_2$ must equal $\sfrac{\alpha}{2}$, and $\alpha_1$ must be between $\sfrac{\alpha}{2}$ and $\alpha$, inclusively.
    Now, we need to demonstrate that there exists an $x \in \mathbb{R}$ where
    \begin{equation}
        D_{\alpha_1,\sfrac{1}{\alpha}}(x) - H_\alpha(x) > 0
    \end{equation}
    when $\sfrac{\alpha}{2} < \alpha_1 < \alpha$.
    The inequality is equal to
    \begin{equation}
        \alpha \left( \alpha_1 \exp\left(-\frac{x}{\alpha_1}\right) + x - \alpha_1 \right) - \frac{1}{2} x^2 > 0
    \end{equation}
    when $0 \leq x \leq \alpha$.
    To simplify the inequality, let us set $\alpha_1 = \sfrac{\alpha}{\gamma}$, substitute $y = \sfrac{\gamma x}{\alpha}$, and divide by $\sfrac{\alpha^2}{\gamma}$ where $1 < \gamma < 2$, which results in the following inequality:
    \begin{equation}
        f_6(y) = \exp(-y) + y - 1 - \frac{1}{2 \gamma} y^2 > 0
        \label{eqn:inequality-6}
    \end{equation}
    when $0 \leq y \leq \gamma$.
    The first and second derivative of $f_6(y)$ are
    \begin{equation}
        f'_6(y) = -\exp(-y) + 1 - \frac{1}{\gamma} y
    \end{equation}
    and
    \begin{equation}
        f''_6(y) = \exp(-y) - \frac{1}{\gamma}\text{.}
    \end{equation}
    At $y = 0$, $f'_6(0) = 0$ and $f''_6(0) = 1 - \sfrac{1}{\gamma} > 0$ for ${1 < \gamma < 2}$, and at $y = \gamma$, $f'_6(\gamma) = -\exp(-\gamma) < 0$.
    Like before, by Rolle's theorem, $f'_6(y)$ can have at most two roots since $f''_6(y)$ has a single root.
    Therefore, there exists a unique value, $0 < y_0 < \gamma$, where $f'_6(y_0) = 0$, and on the interval $0 < y < y_0$, $f'_6(y) > 0$.
    Again, by the mean value theorem, $f_6(y) > 0$ on that interval, $0 < y < y_0$, since $f'_6(y) > 0$ and $f_6(0) = 0$.
    Therefore, $\alpha_1$ must equal $\alpha$ for the original inequalities to hold.
\end{proof}

\newpage
\section{Experimental Validation of Target Approximation}
\label{sec:target-approximation}

In Section~\ref{sec:case-study}, we claim the target width and target height can be approximated with the percentage change between the anchor and the ground-truth.
To validate the approximation, we train the Faster R-CNN model with the following targets:
\begin{equation}
    t^*_w = \frac{w^*}{w_a} - 1
\end{equation}
and
\begin{equation}
    t^*_h = \frac{h^*}{h_a} - 1\text{.}
\end{equation}
No other changes were made to the implementation.
Refer to Section~\ref{sec:experiments-faster}, for details on the training and evaluation procedure.
The results of the experiment are shown in Table~\ref{tab:target-results}.
Only a very slight degradation in performance is observed by replacing the targets with its approximation, which we believe validates our use of the approximation in our interpretation of the loss functions.

\begin{table}[h]
  \centering
  \caption{Target Performance}
  \vspace{-0.5em}
  \scalebox{0.88}{
    \begin{tabular}{c|ccc}
      \multirow{2}{*}{Target} & \multicolumn{3}{c}{Mean Average Precision (mAP) @} \\
      & 0.5 IoU & 0.75 IoU & 0.5-0.95 IoU \\
      \hline
      \hline
      Original & \textbf{44.7} & \textbf{23.1} & \textbf{23.8} \\
      Approximation & 44.6 & 23.0 & 23.7 \\
    \end{tabular}
  }
  \label{tab:target-results}
\end{table}

\section{Evaluation of Proposed Loss Function}
\label{sec:loss-experiment}

Although our goal is to understand the Huber loss and not to replace it, for the sake of completion, we demonstrate that replacing the Huber loss with our proposed loss function produces comparable results.
As mentioned in Section~\ref{sec:relationship}, minimizing the loss function $D_{\sfrac{\alpha}{2},\sfrac{1}{\alpha}}$ is equivalent to minimizing an upper-bound on the Huber loss $H_\alpha$.
Therefore, for this experiment we simply replace $H_\alpha$ with $D_{\sfrac{\alpha}{2},\sfrac{1}{\alpha}}$ with no other modifications to Faster R-CNN.
Refer to Section~\ref{sec:experiments-faster}, for details on the training and evaluation procedure.
The results are shown in Table~\ref{tab:loss-results}.
The performance of the loss functions are nearly identical, which is expected due to the similarity of the functions.

\begin{table}[h]
  \centering
  \caption{Loss Function Performance}
  \vspace{-0.5em}
  \scalebox{0.88}{
    \begin{tabular}{c|ccc}
      \multirow{2}{*}{Loss Function} & \multicolumn{3}{c}{Mean Average Precision (mAP) @} \\
      & 0.5 IoU & 0.75 IoU & 0.5-0.95 IoU \\
      \hline
      \hline
      $H_\alpha$ & \textbf{44.7} & 23.1 & \textbf{23.8} \\
      $D_{\sfrac{\alpha}{2},\sfrac{1}{\alpha}}$ & \textbf{44.7} & \textbf{23.3} & \textbf{23.8}
    \end{tabular}
  }
  \label{tab:loss-results}
\end{table}

\end{document}